\documentclass[twoside]{article}

\usepackage[accepted]{aistats2020}
\usepackage[numbers]{natbib}

\usepackage{amsthm,amsmath} 
\usepackage[colorlinks,linkcolor=red,anchorcolor=blue,citecolor=blue]{hyperref} 
\usepackage[capitalize]{cleveref} 
\usepackage{smile}
\usepackage{algorithm}
\usepackage{extarrows}
\usepackage[OT1]{fontenc}
\usepackage{fullpage}
\usepackage[protrusion=false,expansion=true]{microtype}
\usepackage[algo2e]{algorithm2e} 
\usepackage{graphicx}
\usepackage{comment}
\usepackage{stmaryrd}
\usepackage[dvipsnames]{xcolor}
 
\definecolor{mypink1}{rgb}{0.858, 0.188, 0.478}
\definecolor{mypink2}{RGB}{219, 48, 122}
\definecolor{mypink3}{cmyk}{0, 0.7808, 0.4429, 0.1412}
\definecolor{mygray}{gray}{0.6}

\newcommand{\mR}{\mathbb{R}}

\newcommand{\mI}{\mathbb{I}}

\newcommand{\E}{\mathbb E}

\newcommand{\ip}[1]{\langle #1 \rangle}
\newcommand{\Prob}[1]{\mathbb P\left(#1\right)}

\usepackage[textsize=tiny]{todonotes}
\def\botao{\color{blue}}
\def\red{\color{red}}

\renewcommand{\epsilon}{\varepsilon}

\begin{document}

%

%

\twocolumn[

\aistatstitle{Adaptive Exploration in Linear Contextual Bandit}

\aistatsauthor{ Botao Hao \And Tor Lattimore \And  Csaba Szepesv\'ari}

\aistatsaddress{ Princeton University \And  Deepmind \And Deepmind and University of Alberta} 
]

\begin{abstract}
Contextual bandits serve as a fundamental model for many sequential decision making tasks. 
The most popular theoretically justified approaches are based on the optimism principle.
While these algorithms can
be practical, they are known to be suboptimal asymptotically. 
On the other hand, existing asymptotically optimal algorithms for this problem
do not exploit the linear structure in an optimal way and suffer from lower-order terms that dominate the regret in all practically interesting
regimes. We start to bridge the gap by designing an algorithm that is asymptotically optimal and has good finite-time empirical performance.
At the same time, we make connections to the recent literature on when exploration-free methods are effective. Indeed, if the distribution of contexts
is well behaved, then our algorithm acts mostly greedily and enjoys sub-logarithmic regret. Furthermore, our approach is adaptive in the sense that
it automatically detects the nice case. Numerical results demonstrate significant regret reductions by our method relative to several baselines.
\end{abstract}

\section{INTRODUCTION}\label{sec:intro}

Stochastic contextual linear bandits, the problem we consider, is interesting due to its rich structure and also because of its potential applications, e.g., in online recommendation systems
\citep{agarwal2009online,Li:2010:CAP:1772690.1772758}. 
In this paper we propose a new algorithm for this problem 
that is asymptotically optimal, computationally efficient 
and empirically well-behaved in finite-time regimes. 
As a consequence of asymptotic optimality, the algorithm adapts to easy instances
where it achieves sub-logarithmic regret.

Popular approaches for regret minimisation in contextual bandits 
include $\epsilon$-greedy \citep{langford2007epoch}, explicit optimism-based algorithms \citep{dani2008stochastic,rusmevichientong2010linearly, chu2011contextual, abbasi2011improved}, and implicit ones, such as Thompson sampling \citep{agrawal2013thompson}.
Although these algorithms enjoy near-optimal worst-case guarantees and can be quite practical, they are known to be arbitrarily 
suboptimal in the asymptotic regime, even in the non-contextual linear bandit \citep{lattimore2017end}.

We propose an optimisation-based algorithm that estimates and tracks the optimal allocation for each context/action pair. This technique is most well known
for its effectiveness in pure exploration \citep[and others]{ChanLai06,GK16,DKM19}.
The approach has been used in regret minimisation in linear bandits with fixed action sets \citep{lattimore2017end} and structured bandits \citep{combes2017minimal}.
The last two articles provide algorithms for the non-contextual case and hence cannot be applied directly to our setting.
More importantly, however, the algorithms are not practical. The first algorithm uses a complicated three-phase construction that
barely updates its estimates. The second algorithm is not designed to handle large action spaces and has a `lower-order' term in the regret
that depends linearly on the number of actions and dominates the regret in all practical regimes. This lower-order term is not merely a product of the analysis, but also reflected in the experiments (see Section \ref{sec:large} for details).

The most closely related work is by \citet{ok2018exploration} who study a reinforcement learning setting. 
A stochastic contextual bandit can be viewed as a Markov decision process where the state represents the context and the transition is independent of the action.
The structured nature of the mentioned paper means our setting is covered by their algorithm.
Again, however, the algorithm is too general to exploit the specific structure of the contextual bandit problem. Their algorithm is asymptotically
optimal, but suffers from lower-order terms that are linear in the number of actions and dominate the regret in all practically interesting regimes. In contrast, our algorithm is asymptotically optimal, but also practical in finite-horizon regimes, as will be demonstrated by our experiments.

The contextual linear bandit also serves as an interesting example where the asymptotics of the problem are not indicative of what should
be expected in finite-time (see the second scenario in Section \ref{sec:changing}).
This is in contrast to many other bandit models where the asymptotic regret is also roughly optimal in finite time \citep{lattimore2018bandit}.
There is an important lesson here. Designing algorithms that optimize for the asymptotic regret may make huge sacrifices in finite-time.

Another interesting phenomenon is related to the idea of `natural exploration' that occurs in contextual bandits \citep{bastani2017mostly,kannan2018smoothed}.
A number of authors have started to investigate the striking performance of greedy algorithms in contextual bandits. In most bandit settings the greedy policy
does not explore sufficiently and suffers linear regret. In some contextual bandit problems, however, the changing features ensure the algorithm cannot
help but explore. Our algorithm and analysis highlights this effect (see Section \ref{subsec:sub-log} for details). 
If the context distribution is sufficiently rich, then the algorithm is eventually almost
completely greedy and enjoys sub-logarithmic regret. 
As opposed to the cited previous works, our algorithm achieves this under the cited favourable conditions \emph{while} at the same time it satisfies the standard optimality guarantees when the favourable conditions do not hold.
As another contribution, we prove that algorithms based on optimism, similarly to the new algorithm, also enjoy sub-logarithmic regret in the rich-context distribution setting (\cref{thm:ucb}), and hence differences appear in lower order terms only between these algorithms.

The rest of the paper is organized as follows. 
We first introduce the problem setting (Section \ref{sec:setting}),
which we follow by presenting our asymptotic lower bound (Section \ref{sec:lower_bound}).
Section \ref{sec:algorithm} introduces our new algorithm, which is claimed to match the lower bound. 
A proof sketch of this claim is presented in the same section.
Section \ref{sec:experiments} presents experiments to illuminate the behaviour of the new algorithm in comparison to its strongest competitors.
Section \ref{sec:discussion} discusses remaining notable open questions.

\textbf{Notation} 
Let $[n] = \{1,2, \ldots, n\}$. For a  vector $x$ and positive semidefinite matrix $A$ we let $\|x\|_A=\sqrt{x^{\top}Ax}$. 
The cardinality of a set $\cA$ is denoted by $|\cA|$.

\section{PROBLEM SETTING}\label{sec:setting}

We consider the stochastic $K$-armed contextual linear bandit with a horizon of $n$ rounds and $M$ possible contexts. 
The assumption that the contexts are discrete cannot be dropped but as we shall at least $M$ will not play an important role in the regret bounds.
This assumption would hold for example in a recommender system if users are clustered into finitely many groups.
For each context $m \in [M]$ there is a known feature/action set $\cA^m\subset \mathbb R^d$ with $|\cA^m| = K$.
The interaction protocol is as follows. First the environment samples a sequence of independent contexts $(c_t)_{t=1}^n$ from an unknown distribution $p$ over $[M]$ and each context is assumed to appear with positive probability.
At the start of round $t$ the context $c_t$ is revealed to the learner, who may use their observations to choose an action $X_t \in \cA_t = \cA^{c_t}$.
The reward is
\begin{equation*}
Y_t = \langle X_t,\theta \rangle + \eta_t\,,
\end{equation*}
where $(\eta_t)_{t=1}^n$ is a sequence of independent standard Gaussian random variables and $\theta\in\mathbb R^d$ is an unknown parameter. 
The Gaussian assumption can be relaxed to conditional sub-Gaussian assumption for the regret upper bound, but is necessary for the regret lower bound. 
Throughout, we consider a frequentist setting in the sense that $\theta$ is fixed. 
For simplicity, we assume each $\cA^{m} $ spans $\mathbb R^d$ and $\|x\|_2\leq1$ for all $x\in \cup_{m}\cA^{m}$. 

The performance metric is the cumulative expected regret, which measures the difference between the expected cumulative reward collected by
the omniscient policy that knows $\theta$ and the learner's expected cumulative reward.
The optimal arm associated with context $m$ is $x_m^* = \argmax_{x\in\cA^m}\langle x, \theta\rangle$.
Then the expected cumulative regret of a policy $\pi$ when facing the bandit determined by $\theta$ is
\begin{equation*}
    R_{\theta}^{\pi}(n) = \mathbb E\left[\sum_{t=1}^n\langle x_{c_t}^*, \theta\rangle-\sum_{t=1}^nY_t \right].
\end{equation*}
Note that this cumulative regret also depends on the context distribution $p$ and action sets. They are omitted
from the notation to reduce clutter and because there will never be ambiguity.

\section{ASYMPTOTIC LOWER BOUND}\label{sec:lower_bound}

We investigate the fundamental limit of linear contextual bandit by deriving its instance-dependent asymptotic lower bound.  
First, we define the class of policies that are taken into consideration.

\begin{definition}[Consistent Policy]\label{def:consis_policy}
A policy $\pi$ is called consistent 
if the regret is subpolynomial for any bandit in that class and all context distributions: 
\begin{equation}\label{eqn:consistent_policy}
    R_{\theta}^{\pi}(n) = o(n^\epsilon), \ \text{for all} \ \epsilon>0 \ \text{and all} \ \theta\in\mathbb R^d.
\end{equation}
\end{definition}

The next lemma is the key ingredient in proving the asymptotic lower bound.
Given a context $m$ and $x \in \cA^m$ let $\Delta_{x}^{m} = \langle x_m^*-x, \theta \rangle$ be the suboptimality gap. 
Furthermore, let $\Delta_{\min} = \min_{m\in[M]}\min_{x\in\cA^{m}, \Delta_x^m>0}\Delta_{x}^{m}$.

\begin{lemma}\label{lemma:confidence_width}
Assume that $p(m) > 0$ for all $m \in [M]$ and
that $x_m^*$ is uniquely defined for each context $m$ and let $\pi$ be consistent.
Then for sufficiently large $n$ the expected covariance matrix
\begin{equation}\label{def:covariance_matrix}
    \bar{G}_n = \mathbb E\left[\sum_{t=1}^nX_t X_t^{\top}\right],
\end{equation}
is invertible. Furthermore, for any context $m$ and any arm $x\in\cA^{m}$,
\begin{equation}\label{eqn:weight_norm}
    \limsup_{n\to \infty} \log (n)\big\|x-x^*_m\big\|_{\bar{G}_n^{-1}}^2\leq \frac{(\Delta_{x}^{m})^2}{2} \,.
\end{equation}
\end{lemma}
The proof is deferred to Appendix \ref{sec:lemma_confidence} in the supplementary material. 
Intuitively, the lemma shows that any consistent policy must collect sufficient statistical evidence at confidence level $1 - 1/n$
that suboptimal arms really are suboptimal.
This corresponds to ensuring that the width of an appropriate confidence interval
$\sqrt{2\log (n)}\|x-x^*_m\|_{\bar{G}_n^{-1}}$ is approximately smaller than the sub-optimality gap $\Delta_x^m$.

\begin{theorem}[Asymptotic Lower Bound]\label{thm:lower_bound}
Under the same conditions as \cref{lemma:confidence_width},
\begin{equation}\label{eqn:lower_bound}
    \liminf_{n\to\infty}\frac{R_{\theta}^{\pi}(n)}{\log(n)}\geq \cC(\theta,\cA^1,\ldots, \cA^M) \,,
\end{equation}
where $\cC(\theta, \cA^1,\ldots, \cA^M)$ is defined as the optimal value of the following optimisation problem: 
\begin{equation}\label{eqn:opti_problem}
    \inf_{\alpha_{x,m}\in[0, \infty]}\sum_{m=1}^M\sum_{x\in\cA^m}\alpha_{x,m}\Delta_x^m 
\end{equation}
subject to the constraint that for any context $m$ and suboptimal arm $x\in \cA^{m}$,
\begin{equation}\label{eqn:constraint}
    x^{\top}\left(\sum_{m=1}^M\sum_{x\in\cA^m}\alpha_{x,m}xx^{\top}\right)^{-1}x \leq \frac{(\Delta_x^m)^2}{2} \,.
\end{equation}
\end{theorem}
Given the result in Lemma \ref{lemma:confidence_width}, the proof of Theorem \ref{thm:lower_bound} follows exactly the same idea of the proof of Corollary 2 in \cite{lattimore2017end} and thus is omitted here. Later on we will prove a matching upper bound in Theorem \ref{thm:upper_bound} and argue that our asymtotical lower bound is sharp.

\begin{remark}
In the above we adopt the convention that $\infty \times 0 = 0$ so that $\alpha_{x,m} \Delta_x^m = 0$ whenever $\Delta_x^m = 0$.
The inverse of a matrix with infinite entries is defined by passing to the limit in the obvious way, and is not technically an inverse.
\end{remark}

\begin{remark}\label{remark:allocation}
Let us denote $\{\alpha_{x,m}^{*}\}_{x\in\cA^{m}, m\in[M]}$ as an optimal solution to the above optimisation problem. 
It serves as the \emph{optimal allocation rule} for each arm such that the cumulative regret is minimized subject to the width
of the confidence interval of each sub-optimal arm is small. Specifically, $\alpha_{x,m}^*\log(n)$ can be interpreted as the 
approximate optimal number of times arm $x$ should be played having observed context $m$.
\end{remark}

\begin{remark}
Our lower bound may also be derived from a more general bound of \citet{ok2018exploration}, since a stochastic contextual bandit can be viewed as a kind of Markov decision process. We use an alternative proof technique and the two lower bound statements have different forms. The proof is included for completeness.
\end{remark}

\begin{example}
When $M=1$ and $\cA^1=\{e_1, \ldots, e_d\}$ is the standard basis vectors, the problem reduces to classical 
multi-armed bandit and $\cC(\theta, \cA^1) = \sum_{x \in \cA^1, \Delta_x > 0} 2/\Delta_x$, which
matches the well-known asymptotic lower bound by \cite{lai1985asymptotically}.
\end{example}

The constant $\cC(\theta, \cA^1, \ldots, \cA^M)$ depends on both the unknown parameter $\theta$ and the action sets $\cA^1,\ldots,\cA^M$, but \textit{not}
the context distribution $p$. In this sense there is a certain discontinuity in the hardness measure $\cC$ as a function of the context distribution.
More precisely, problems where $p(m)$ is arbitrarily close to zero may have different regret asymptotically than the problem obtained by removing context $m$ entirely.
Clearly as $p(m)$ tends to zero the $m$th context is observed with vanishingly small probability in finite time 
and hence the asymptotically optimal regret may not be representative of the finite-time hardness. 

\subsection{Sub-logarithmic regret} \label{subsec:sub-log}
Our matching upper and lower bounds reveal the
interesting phenomenon that if the action sets satisfy certain conditions, then sub-logarithmic regret is possible. 
Consider the scenario that the set of optimal arms $\{x_1^*, \ldots, x_M^*\}$ spans $\mathbb R^d$ \footnote{This condition is both sufficient and necessary. More precisely, sub-logarithmic regret is possible if and only if the optimal arms span the space that is spanned by all the available actions. Since in this paper we assume the action set spans $\mathbb R^d$ for simplicity, the two conditions are equivalent.}. 
Let $\Lambda \in \mathbb R$ be a large constant to be defined subsequently and for each context $m$ and arm $x \in \cA^m$, let $\alpha_{x,m}$ be 0 if $x\neq x_m^*$, and be $\Lambda$ if else.
Then,
\begin{equation}\label{eqn:eqn1}
    \sum_{m=1}^M\sum_{x\in\cA^m}\alpha_{x,m} x x^{\top} = \Lambda \sum_{m=1}^M x_m^*x_m^{*\top} \,.
\end{equation}
Since the set of optimal arms spans $\mathbb R^d$ it holds that for any context $m$ and arm $x\in\cA^m$,
\begin{equation}\label{eqn:eqn2}
   x^{\top}\left(\sum_{m=1}^Mx_m^*x_m^{*\top}\right)^{-1} x < \infty \,.
\end{equation}
Combining \cref{eqn:eqn1} and \cref{eqn:eqn2},
\begin{align*}
&x^{\top}\left(\sum_{m=1}^M\sum_{x\in\cA^m}\alpha_{x,m} x x^{\top}\right)^{-1} x \\
    &\qquad= \Lambda^{-1}x^{\top}\left(\sum_{m=1}^Mx_m^*x_m^{*\top}\right)^{-1}x \,.
\end{align*}
Hence, the constraint in \cref{eqn:constraint} is satisfied for sufficiently large $\Lambda$. 
Since with this choice of $(\alpha_{x,m})$ we have $\sum_{m=1}^M\sum_{x\in\cA^m}\alpha_{x,m}\Delta_x^m =0$, it follows that $\cC(\theta, \cA^1,\ldots, \cA^M) = 0$. 
Therefore our upper bound will show that when the set of optimal actions
$\{x_1^*, \ldots, x_M^*\}$ spans $\mathbb R^d$ our new algorithm satisfies
\begin{equation*}
    \liminf_{n\to\infty}\frac{R_{\theta}^{\pi}(n)}{\log(n)} = 0 \,.
\end{equation*}
\begin{remark}\label{remark:sub_log}
The choice of $\alpha_{x,m}$ above shows that when $\{x_1^*, \ldots, x_M^*\}$ span $\mathbb R^d$, then an asymptotically optimal algorithm only needs to play suboptimal arms
sub-logarithmically often, which means the algorithm is eventually very close to the greedy algorithm. \citet{bastani2017mostly,kannan2018smoothed} also investigate the striking performance of greedy algorithms in contextual bandits. However, \cite{bastani2017mostly} assume the covariate diversity on the context distribution while \cite{kannan2018smoothed} assume the context is artificially perturbed with noise
-- these assumptions make these works brittle. 
In addition, \cite{bastani2017mostly} only provide a rate-optimal algorithm while our algorithm is optimal in constants (see Theorem \ref{thm:upper_bound} for details).
\end{remark}


As claimed in the introduction, we also prove that algorithms based on optimism can enjoy bounded regret when
the set of optimal actions spans the space of all actions. The proof of the following theorem is given in \cref{sec:thm:ucb}.

\begin{theorem}\label{thm:ucb}
Consider the policy $\pi$ that plays optimistically by
\begin{align*}
X_t = \argmax_{x \in \cA^{c_t}} \ip{\hat \theta_{t-1}, x} + \|x\|_{G_t^{-1}} \beta_t^{1/2}\,.
\end{align*}
Suppose that $\theta$ is such that $\{x_1^*, \ldots, x_M^*\}$ spans $\mathbb R^d$. Then, 
for suitable $(\beta_t)_{t=1}^n$ with $\beta_t = O(d \log(t))$, it holds that
$\limsup_{n\to\infty} R_\theta^\pi(n) < \infty$.
\end{theorem}

Note, the choice of $(\beta_t)$ for which the above theorem holds also guarantees the standard $\tilde O(d \sqrt{n})$ minimax
bound for this algorithm, showing that LinUCB can adapt online to this nice case.

\section{OPTIMAL ALLOCATION MATCHING}\label{sec:algorithm}
The instance-dependent asymptotic lower bound provides an optimal allocation rule.
However, the optimal allocation $\{\alpha_{x,m}^*\}_{x,m}$ depends on the unknown sub-optimality gap. In this section, we present a novel matching 
algorithm that simultaneously estimates the unknown parameter $\theta$ using least squares and updates the allocation rule.

\subsection{Algorithm}\label{sec:alg}
Let $N_x(t) = \sum_{s=1}^t\mI(X_s=x)$ be the number of pulls of arm $x$ after round $t$ 
and $G_t = \sum_{s=1}^tX_sX_s^{\top}$. The least squares estimator is $\hat{\theta}_t = G_t^{-1}\sum_{x=1}^t X_sY_s$.
For each context $m$ the estimated sub-optimality gap of arm $x \in \cA^m$ is $\hat{\Delta}_x^m(t) = \max_{y\in\cA^m}\langle y - x, \hat{\theta}_t\rangle$ and
the estimated optimal arm is $\hat{x}_m^*(t) = \argmax_{x\in\cA^m}\langle x, \hat{\theta}_t\rangle$.
The minimum nonzero estimated gap is 
\begin{align*}
\hat{\Delta}_{\min}(t)=\min_{m\in[M]}\min_{x\in\cA^m, \hat{\Delta}_x^m(t)>0}\hat{\Delta}_x^m(t) \,.
\end{align*}

Next, we define a similar optimisation problem as  in \eqref{eqn:opti_problem} but with a different normalisation.

\begin{definition}\label{def:optimi}
Let $f_{n,\delta}$ be the constant given by
\begin{equation}\label{def:ft}
    f_{n,\delta} = 2(1+1/\log(n))\log(1/\delta) + cd\log(d\log(n))\,,
\end{equation}
where $c$ is an absolute constant. We write $f_n=f_{n, 1/n}$. For any $\tilde{\Delta}\in[0, \infty)^{|\cup_m\cA^m|}$ define $T(\tilde{\Delta})$ as a solution of the following optimisation problem:
\begin{equation}\label{eqn:opti}
    \min_{(T_x^m)_{x,m} \in[0, \infty]}\sum_{m=1}^M\sum_{x\in\cA^m}T_x^m\tilde{\Delta}_x^m,
\end{equation}
subject to
\begin{equation*}
    \|x\|_{H_T^{-1}}^2\leq \frac{\Delta_x^2}{f_n}, \forall x\in\cA^{m}, m\in[M].
\end{equation*}
and that $H_T = \sum_{m=1}^M\sum_{x\in\cA^m}T_x^m xx^{\top}$ is invertible. 
\end{definition}
 
If $\tilde{\Delta}$ is an estimate of $\Delta$, we call the solution $T(\tilde{\Delta})$ an \emph{approximated allocation rule} 
in contrast to the \emph{optimal allocation rule} defined in Remark \ref{remark:allocation}.  Our algorithm alternates 
between exploration and exploitation, depending on whether or not all the arms have satisfied the approximated allocation rule.
We are now ready to describe the algorithm, which starts with a brief initialisation phase.

\paragraph{Initialisation}
In the first $d$ rounds the algorithm chooses any action $X_t$ in the action set such that $X_t$ is not in the span of $\{X_1,\ldots,X_{t-1}\}$. This is always
possible by the assumption that $\cA^m$ spans $\mathbb R^d$ for all contexts $m$.
At the end of the initialisation phase $G_t$ is guaranteed to be invertible.

\paragraph{Main phase}
In each round
after the initialisation phase the algorithm checks if the following criterion holds for any $x \in \cA^{c_t}$: 
\begin{equation}\label{eqn:criterion}
   \|x\|_{G_{t-1}^{-1}}^2\leq \max\Big\{\frac{(\hat{\Delta}_{\min}(t-1))^2}{f_n}, \frac{(\hat{\Delta}_x^{c_t}(t-1))^2}{f_n}\Big\}.
\end{equation}
The algorithm exploits if \cref{eqn:criterion} holds and explores otherwise, as explained below.
 
\paragraph{Exploitation.} 
The algorithm exploits by taking the greedy action:
\begin{equation}
    X_t=\argmax_{x\in\cA^{c_t}}x^{\top}\hat{\theta}_{t-1}.
\end{equation}
\paragraph{Exploration.} 
The algorithm explores when \cref{eqn:criterion} does not hold. This means that some actions have not been explored sufficiently.
There are two cases to consider. First, when there exists an arm $x' \in \cA^{c_t}$ such that 
\begin{align*}
N_{x'}(t-1)< \min(T_{x'}^{c_t}(\hat{\Delta}(t-1)), f_n/\hat{\Delta}_{\min}^2(t-1)) ,
\end{align*}
the algorithm then computes two actions 
\begin{align}
b_1 &= \argmin_{x\in \cA^{c_t}}\frac{N_x(t-1)}{\min(T_x^{c_t}(\hat{\Delta}(t-1)),f_n/\hat{\Delta}_{\min}^2(t-1))} \nonumber \\
b_2 &= \argmin_{x\in\cA^{c_t}}N_x(t-1).
\label{eqn:exploration}
\end{align}
Let $s(t)$ be the number of exploration rounds defined in Algorithm \ref{alg:main}.
If $N_{b_2}(t-1)\leq \varepsilon_t s(t)$ the algorithm plays arm $X_t = b_2$ -- a form of forced exploration. Otherwise the algorithm plays arm $X_t = b_1$.
Finally, rounds where an $x' \in \cA^{c_t}$ with the required property does not exist are called \textit{wasted}. In these rounds
the algorithm acts optimistically as LinUCB \citep{abbasi2011improved}: 
\begin{equation}\label{eqn:linucb}
    X_t=\argmax_{x\in\cA^{c_t}}x^{\top}\hat{\theta}_{t-1} + \sqrt{f_{n, 1/(s(t))^2}}\|x\|_{G_{t-1}^{-1}},
\end{equation}
where $f_{n, 1/(s(t))^2}$ is defined in \cref{def:ft}. The complete algorithm is presented in Algorithm \ref{alg:main}.

\begin{remark}
The naive forced exploration can be improved by calculating a barycentric spanner \citep{awerbuch2008online} for each action set and then playing the least played action in the spanner.
In normal practical setups this makes very little difference, where the forced exploration plays a limited role. For finite-time worst-case analysis, however, it may
be crucial, since otherwise the regret may depend linearly on the number of actions, while using the spanner guarantees the forced exploration is sample efficient.
\end{remark}

\subsection{Asymptotic Upper Bound}

Our main theorem states that \cref{alg:main} is asymptotically optimal under mild assumptions. 
\begin{theorem}\label{thm:upper_bound}
Suppose that $T_x^m(\Delta)$ is uniquely defined and $T_x^m(\cdot)$ is continuous at $\Delta$ for all contexts $m$ and actions $x \in \cA^m$.
Then the policy $\pi_{\text{oam}}$ proposed in Algorithm \ref{alg:main} with $\epsilon_t = 1/\log(\log(t))$ satisfies
\begin{equation}
    \limsup_{n\to \infty}\frac{R_{\theta}^{\pi_{\text{oam}}}(n)}{\log (n)}\leq \cC(\theta, \cA^1,\ldots, \cA^M).
\end{equation}
\end{theorem}
Together with the asymptotic lower bound in Theorem \ref{thm:lower_bound}, we can argue that optimal-allocation matching algorithm is asymptotical optimal and the lower bound in \cref{eqn:lower_bound} is sharp. 
\begin{remark}
The assumption that $T_x^m(\cdot)$ is continuous at $\Delta$ is used to ensure the stability of our algorithm.
We prove that the uniqueness assumption actually implies continuity (\cref{lem:cont} in the supplementary material) and thus the continuity assumption could be omitted. There are, however, certain corner cases where uniqueness does not hold. For example when
$\theta = (1,0)^{\top}, \cA = \{(1, 0), (0, 1), (0, -1)\}$. 
\end{remark}
\begin{algorithm}[t]
\SetAlgoLined
\KwIn{
 exploration parameter $\varepsilon_t$, exploration counter $s(d)=0$.
}
{\red \# \emph{initialisation}}

 \For{$t = 1$ \rm{\textbf{to}} $d$}{
 Observe an action set $\cA^{c_t}$, pull arm $X_t$ such that $X_t$ is not in the span of $\{X_1,\ldots, X_{t-1}\}$.
 }

  \For{$t = d+1$ \rm{\textbf{to}} $n$}{
  
  Observe an action set $\cA^{c_t}$ and solve the optimisation problem \eqref{eqn:opti} based on the estimated gap $\hat{\Delta}(t-1)$. 
  
  \eIf{$   \|x\|_{G_{t-1}^{-1}}^2\leq \max\{\frac{\hat{\Delta}_{\min}^2(t-1)}{f_n}, \frac{(\hat{\Delta}_x^{c_t}(t-1))^2}{f_n}\},  \forall x \in\cA^{c_t}$,}{
  
  {\red \# \emph{exploitation}}
  
   Pull arm $X_t=\argmax_{x\in\cA^{c_t}}x^{\top}\hat{\theta}_{t-1}$.
   
   }{
   
   {\red \# \emph{exploration}}
   
   $s(t) = s(t-1)+1$
   
   \eIf{$N_x(t-1)\geq \min(T_x(\hat{\Delta}(t-1)), f_n/(\hat{\Delta}_{\min}(t-1)))^2, \forall x \in\cA^{c_t}$,}{
   
   Pull arm according to LinUCB in \eqref{eqn:linucb}.

   }{
  Calculate $b_1,b_2$ as in \cref{eqn:exploration}.
   
   \eIf{$N_{b_2}(t-1)\leq \varepsilon_t s(t-1)$}{
   
   Pull arm $X_t = b_2$. 
   }{
   Pull arm $X_t = b_1$.
   }
   }
  }
  Update $\hat{\theta}_t, \hat{\Delta}^{c_t}_x(t), \hat{\Delta}_{\min}(t)$.
  }
\caption{Optimal Allocation Matching}
\label{alg:main}
\end{algorithm}

\subsection{Proof Sketch}
The complete proof is deferred to Appendix \ref{subsec:proof_upper_bound} in the supplementary material. 
At a high level the analysis of the optimisation-based approach consists of three parts. (1) Showing that the algorithm's estimate of the true parameter is close to the truth
in finite time. (2) Showing that the algorithm subsequently samples arms approximately according to the unknown optimal allocation and (3) Showing that the greedy action
when arms have been sampled sufficiently according to the optimal allocation is optimal with high probability. Existing optimisation-based algorithms suffer from dominant `lower-order' terms
because they use simple empirical means for Part (1), while here we use the data-efficient least-squares estimator.

We let $\textrm{Explore} = \textrm{F-Explore} \cup \textrm{UW-Explore}\cup \textrm{W-Explore}$ be the set of exploration rounds, decomposed into disjoint sets
of forced exploration $(X_t = b_1)$, unwasted exploration $(X_t = b_2)$ and wasted exploration (LinUCB), and let $\textrm{Exploit}$ be the set of exploitation rounds. 

\paragraph{Regret while exploiting}
The criterion in \cref{eqn:criterion} guarantees that the greedy action is optimal with high probability in exploitation rounds.
To see this, note that if $t$ is an exploitation round, then the sub-optimality gap of greedy action $X_t$ satisfies the following with high probability:
\begin{align*}
\Delta_{X_t}^{c_t} \lesssim \sqrt{\frac{\log(\log(n))}{1 \vee N_{X_t}(t-1)}}< \Delta_{\min}\,.
\end{align*}
Since the instantaneous regret either vanishes or is larger than $\Delta_{\min}$, we have $$
\E\left[\sum_{t \in \textrm{Exploit}} \Delta_t^{c_t}\right] = o(\log(n)).$$

\paragraph{Regret while exploring}
Based on the design of our algorithm,
the regret while exploring is decomposed into three terms, 
\begin{equation*}
\begin{split}
\E&\left[\sum_{t \in \textrm{Explore}} \Delta_{X_t}^{c_t}\right]= \E\left[\sum_{t \in \textrm{F-Explore}} \Delta_{X_t}^{c_t}\right]\\
&+ \E\left[\sum_{t \in \textrm{W-Explore}} \Delta_{X_t}^{c_t}\right]+ \E\left[\sum_{t \in \textrm{UW-Explore}} \Delta_{X_t}^{c_t}\right] \,.
\end{split}
\end{equation*}
Shortly we argue that the regret incurred in $\textrm{W-Explore} \cup \textrm{UW-Explore}$ is at most logarithmic and hence
the regret in rounds associated with forced exploration is sub-logarithmic:
\begin{equation*}
    \E\left[\sum_{t \in \textrm{F-Explore}} \Delta_{X_t}^{c_t}\right] = O(\epsilon_n |\textrm{Explore}|) = o(\log(n)) \,.
\end{equation*}
The regret in W-Explore is also sub-logarithmic. To see this, we first argue that $|\textrm{W-Explore}| = O(|\textrm{UW-Explore}|)$ 
since each context has positive probability. Combining with the fact that $|\textrm{UW-Explore}|$ is logarithmic in $n$ and the 
regret of LinUCB is square root in time horizon,
\begin{equation*}
    \E\left[\sum_{t \in \textrm{W-Explore}} \Delta_t^{c_t}\right]= o(\log(n)) \,.
\end{equation*}
The regret in UW-Explore is logarithmic in $n$ with the asymptotically optimal constant using the definition of the optimal allocation:
\begin{align*}
&\limsup_{n\to\infty}\frac{\E\left[\sum_{t \in \textrm{UW-Explore}}  \Delta_t^{c_t}\right]}{\log(n)}= \cC(\theta, \cA^1,\ldots,\cA^M) \,.
\end{align*}
Of course many details have been hidden here, which are covered in detail in the supplementary material.

\section{EXPERIMENTS}\label{sec:experiments}
In this section, we first compare our proposed algorithm and LinUCB \citep{abbasi2011improved} on some specific problem instances to showcase their strengths and weaknesses. 
We examine OSSB  \citep{combes2017minimal} on instances with large action sets to illustrate its weakness due to not using the linear structure everywhere. 
Since \citet{combes2017minimal} demonstrated that OSSB dominates  the algorithm of \citet{lattimore2017end}, we omit this algorithm from our experiments.  In the end, we include the comparison with LinTS \citep{agrawal2013thompson}. Some additional experiments are deferred to Appendix \ref{sec:add_exp} in the supplementary material.

To save computation, we follow the lazy-update approach, similar to that proposed in Section 5.1 of \citep{abbasi2011improved}: The idea is to resolve the optimisation problem \eqref{eqn:opti} whenever $\text{det}(G_t)$ increases by a
constant factor $(1+\zeta)$ and in all scenarios we choose (the arbitrary value) $\zeta = 0.1$. All codes were written in Python. To solve the convex optimisation problem \eqref{eqn:opti}, we use the CVXPY library \citep{cvxpy}.

\subsection{Fixed Action Set}\label{sec:exper_fixed}

Finite-armed linear bandits with fixed action set are a special case of linear contextual bandits. Let $d=2$ and let the true parameter be $\theta= (1,0)^{\top}$. The action set $\cA = \{x_1,x_2,x_3\}$ is fixed and $x_1 = (1, 0)^{\top}$, $x_2 = (0, 1)^{\top}$, $x_3 = (1-u, 5u)^{\top}$. We consider $u=\{0.1,0.2\}$. By construction,  $x_1$ is the optimal arm.
From Figure \ref{plt_fixed}, we observe that LinUCB suffers significantly more regret than our algorithm. The reason is that if $u$ is very small, then $x_1$ and $x_3$ point in almost the same direction and so choosing only these arms does not provide sufficient
information to quickly learn which of $x_1$ or $x_3$ is optimal. On the other hand, $x_2$
and $x_1$ point in very different directions and so choosing $x_2$ allows a learning
agent to quickly identify that $x_1$ is in fact optimal. LinUCB stops pulling $x_2$ once it is optimistic and thus fails to find the right balance between information and reward. Our algorithm, however, takes this into consideration by tracking the optimal allocation ratios.
	\vspace{-0.15in}
 \begin{figure}[h]
	\centering
		\includegraphics[scale=0.27]{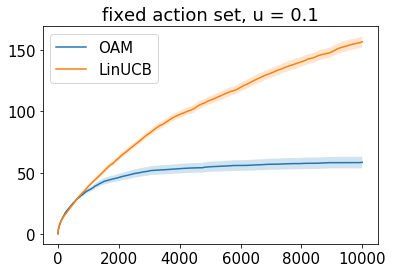}
		\includegraphics[scale=0.27]{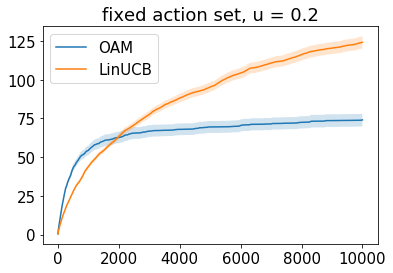}

		\vspace{-0.13in}
\caption{Fixed action set. The results are averaged over 100 realisations. Here and also later, the shaded areas show the standard errors.}\label{plt_fixed}
\end{figure}

\vspace{-0.2in}

\subsection{Changing Action Set}\label{sec:changing}

We consider a simple but representative case when there are only two action sets $\cA^1$ and $\cA^2$ available.  

\textbf{Scenario One.} In each round, $\cA^1$ is drawn with probability 0.3 while $\cA^2$ is drawn with probability 0.7. Set $\cA^1$ contains $x_1^1 = (1, 0, 0)^{\top}$, $x_2^1 = (0, 1, 0)^{\top}$, and $x_3^1 = (0.9, 0.5, 0)^{\top}$, while set $\cA^2$ contains $x_1^2 = (0, 1, 0)^{\top}$, $x_2^2 = (0, 0, 1)^{\top}$, and $x_3^2 = (0, 0.5, 0.9)^{\top}$. The true parameter $\theta$ is $(1, 0, 1)^{\top}$. From the left panel of Figure \ref{plt_changing}, we observe that LinUCB, while starts better, eventually again suffers more regret than our algorithm.
 \begin{figure}[t]
	\centering
		\includegraphics[scale=0.27]{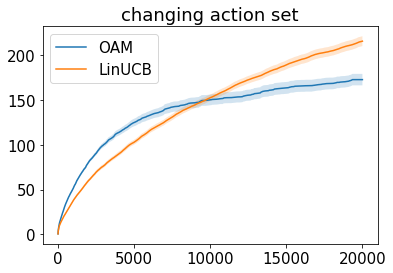}
		\includegraphics[scale=0.27]{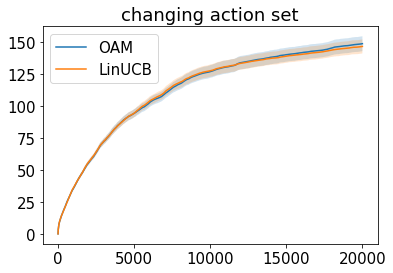}
		\vspace{-0.13in}
\caption{Changing action sets. The left panel is for scenario one and the right panel is for scenario two. The results are averaged over 100 realisations. }\label{plt_changing}
\end{figure}

\textbf{Scenario Two.} In each round, $\cA^1$ is drawn with probability $0.99$, while $\cA^2$ is drawn with probability $0.01$. Set $\cA^1$ contains three actions: $x_1^1 = (1, 0)^{\top}$, $x_2^1 = (0, 1)^{\top}$, $x_3^1 = (0.9, 0.5)^{\top}$, while set $\cA^2$ contains three actions: $x_1^2 = (0, 1)^{\top}$, $x_2^2 = (-1, 0)^{\top}$, $x_3^2 = (-1, 0)$. Apparently, $x_1^1$ and $x_1^2$ are the optimal arms for each action set and they span $\mathbb R^2$. Based on the allocation rule in Section \ref{subsec:sub-log}, the algorithm is advised to pull actions $x_1^1$ and $x_1^2$ very often based on asymptotics. However, since the probability that $\cA^2$ is drawn is extremely small, we are very likely to fall back to wasted exploration and use LinUCB to explore. Thus, in the short term, our algorithm will suffer from the drawback that optimistic algorithms also suffer from and what is described in Section \ref{sec:exper_fixed}. Although, the asymptotics will eventually ``kick in'', it may take extremely long time to see the benefits of this and the algorithm's finite-time performance will be poor.  Indeed, this is seen on the right panel of Figure \ref{plt_changing}, which shows that in this case LinUCB and our algorithm are nearly indistinguishable.

\vspace{-0.1in}

\subsection{Sublinear/Bounded Regret}
Earlier we have argued that when the optimal arms of all action sets span $\mathbb R^d$, our algorithm achieves sub-logarithmic regret. Here, we experimentally study this interesting case.
We consider $M=2$. In each round, $\cA^1$ is drawn with probability 0.8 while $\cA^2$ is drawn with probability 0.2 and the true parameter $\theta$ is $(1, 0)^{\top}$.Set $\cA^1$ contains three actions: $x_1^1 = (1, 0)^{\top}$, $x_2^1 = (0, 1)^{\top}$, $x_3^1 = (0.9, 0.5)^{\top}$, while set $\cA^2$ contains three actions: $x_1^2 = (0, 1)^{\top}$, $x_2^2 = (-1, 0)^{\top}$, $x_3^2 = (-1, 0)$. As discussed before,  $x_1^1$ and $x_1^2$ are the optimal arms for each action set and they span $\mathbb R^2$. The results are shown in the left subpanel of Figure \ref{plt_bounded}. 
The regret of our algorithm appears to have stopped growing after a short period of increase.
In line with \cref{thm:ucb}, LinUCB is seen to achieve bounded regret in this problem.
 \begin{figure}[t]
	\centering
		\includegraphics[scale=0.27]{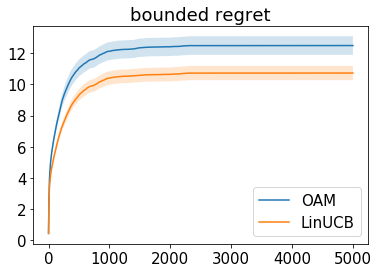}\includegraphics[scale=0.27]{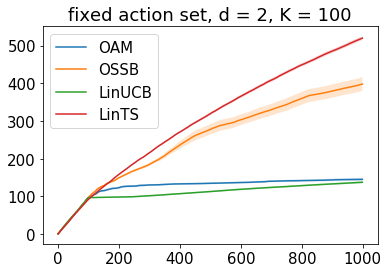}
		\vspace{-0.13in}
\caption{The left panel is for bounded regret and right panel is for large action space. The results are averaged over 100 realisations. }\label{plt_bounded}
\end{figure}
\vspace{-0.13in}
\subsection{Large, Fixed Action Set}\label{sec:large}
We let $d=2$ and $\theta = (1, 0)^{\top}$. 
We generate $100$ uniformly distributed on the $d$-dimensional unit sphere (fixed action set). The results are shown in the right subfigure of Figure \ref{plt_bounded}. When the action space is large, OSSB suffers significantly large regret and becomes unstable due to not using the linear structure everywhere. The regret of (the theoretically justified version of) LinTS is also very large due to the unnecessary variance factor required by its theory.

\section{DISCUSSION}\label{sec:discussion}

We presented a new optimisation-based algorithm for linear contextual bandits that is asymptotically optimal and adapts to both the action sets and unknown parameter. The new algorithm enjoys sub-logarithmic regret when the collection of optimal actions spans $\mathbb R^d$, a property that we also prove for optimism-based approaches. 
There are many open questions. A natural starting point is to prove near-minimax optimality of the new algorithm, possibly with minor modifications. Our work also highlights the dangers of focusing too intensely on asymptotics, which for contextual bandits hide completely the dependence on the context distribution. This motivates the intriguing challenge to understand the finite-time instance-dependent regret. Another open direction is to consider the asymptotics when the context space is continuous, which has not seen any attention.

\subsubsection*{Acknowledgements}
Csaba Szepesv\'ari gratefully  acknowledges  funding  from 
the Canada CIFAR AI Chairs Program, Amii and NSERC.

{
\bibliography{ref}
}

\newpage
\clearpage

\onecolumn

\setlength{\footskip}{55pt}
\setcounter{page}{1}

\appendix
\begin{center}
    \Large Supplement to ``Adaptive Exploration in Linear Contextual Bandit''
\end{center}

In Section \ref{sec:main_proof}, we provide main proofs for asymptotic lower bound and upper bound. In Section \ref{sec:proof:main_lemmas}, we prove several main lemmas. In Section \ref{sec:supporting_lemmas}, some supporting lemmas are presented for the sake of completeness. 

\section{Proofs of Asymptotic Lower and Upper Bounds}\label{sec:main_proof}

First of all, we define the sub-optimal action set as $\cA_{-}^m = \cA^m\setminus \{x:\Delta_x^m=0\}$ and denote $\cA = \cup_{m=1}^M\cA^m$ and $\cA_{-}=\cup_{m=1}^M\cA_{-}^m$.

\subsection{Proof of Lemma \ref{lemma:confidence_width}} \label{sec:lemma_confidence}
The proof idea follows if $\bar{G}_n$ is not sufficiently large in every direction, then some alternative parameters are not sufficiently identifiable. 
\paragraph{Step One.} We fix a consistent policy $\pi$ and fix a context $m\in[M]$ as well as a sub-optimal arm $x\in\cA^m_{-}$. Consider another parameter $\tilde{\theta}\in\mathbb R^d$ such that it is close to $\theta$ but $x_m^*$ is not the optimal arm in bandit $\tilde{\theta}$ for action set $\cA^m$. Specifically, we construct
\begin{equation*}
    \tilde{\theta} = \theta + \frac{H(x-x_m^*)}{\|x-x_m^*\|_H^2}(\Delta_{x}^{m}+\varepsilon),
\end{equation*}
where $H\in\mathbb R^{d\times d}$ is some positive semi-definite matrix and $\varepsilon>0$ is some absolute constant that will be specified later. Since the sub-optimality gap $\tilde{\Delta}_{x_m^*}^m$ satisfies
\begin{equation}
    \langle x-x_m^*, \tilde{\theta}\rangle = \langle x - x_m^*, \theta\rangle + \Delta_{x}^{m} + \varepsilon =\varepsilon>0,
\end{equation}
 it ensures that $x_m^*$ is $\varepsilon$-suboptimal in bandit $\tilde{\theta}$.

We define $T_x(n)=\sum_{t=1}^n\mI(X_t=x)$ and let $\mathbb P$ and $\tilde{\mathbb P}$ be the measures on the sequence
of outcomes $(X_1, Y_1, \ldots, X_n, Y_n)$ induced by the interaction between the policy
and the bandit $\theta$ and $\tilde{\theta}$ respectively. By the definition of $\bar{G}_n$ in \eqref{def:covariance_matrix}, we have 
\begin{eqnarray*}
    \frac{1}{2}\|\theta-\tilde{\theta}\|^2_{\bar{G}_n} &=& \frac{1}{2}(\theta-\tilde{\theta})^{\top}\bar{G}_n(\theta-\tilde{\theta})\\
&=&\frac{1}{2}(\theta-\tilde{\theta})^{\top}\mathbb E\Big[\sum_{x\in\cA}T_x(n)xx^{\top}\Big](\theta-\tilde{\theta})\\
&=&\frac{1}{2}\sum_{x\in\cA}\mathbb E\Big[T_x(n)\Big]\langle x, \theta-\tilde{\theta} \rangle^2.
\end{eqnarray*}
Applying the Bretagnolle-Huber inequality inequality in Lemma \ref{lem:kl} and divergence decomposition lemma in Lemma \ref{lem:inf-processing}, it holds that for any event $\cD$,
\begin{eqnarray}\label{eqn:lower_bound1}
\frac{1}{2}\|\theta-\tilde{\theta}\|^2_{\bar{G}_n}
= \text{KL}(\mathbb P, \tilde{\mathbb P})\geq \log \Big(\frac{1}{2(\mathbb P(\cD)+\tilde{\mathbb P}(\cD^c))}\Big).
\end{eqnarray}

\paragraph{Step Two.} In the following, we start to derive a lower bound of $R_{\theta}^{\pi}(n)$,
\begin{eqnarray*}
R_{\theta}^{\pi}(n) &=&\mathbb E\Big[\sum_{t=1}^n\langle x_{c_t}^*-X_t, \theta\rangle\Big]= \mathbb E\Big[\sum_{m=1}^M \sum_{t:c_t=m}\langle x_m^*-X_t, \theta\rangle\Big]\\
&\geq& \mathbb E\Big[\sum_{t:c_t=m}\langle x_m^*-X_t, \theta\rangle\Big] = \mathbb E\Big[\sum_{t:c_t=m}\Delta_{X_t}^m\Big]\\
&\geq& \Delta_{\min} \mathbb E\Big[\sum_{t:c_t=m}\mathbb I(X_t\neq x_m^*)\Big]= \Delta_{\min} \mathbb E \Big[\sum_{t=1}^n\mathbb I(c_t=m)- \sum_{t=1}^n\mathbb I(c_t=m)\mathbb I(X_t = x_m^*)\Big],
\end{eqnarray*}
where the first inequality comes from the fact that $\langle x_m^*-X_t, \theta\rangle\geq 0$ for all $m\in[M]$. Define the event $\cD$ as follows,
\begin{equation}\label{def:event_D}
    \cD = \Big\{\sum_{t=1}^n \mathbb I(c_t=m)\mathbb I (X_t=x_m^*)\leq \frac{1}{2}\sum_{t=1}^n\mathbb I(c_t=m)\Big\}.
\end{equation}
When event $\cD$ holds, we will only pull at most half of total rounds for the optimal action of action set $m$. Then it holds that
\begin{eqnarray*}
R_{\theta}^{\pi}(n)&\geq& \Delta_{\min} \mathbb E \Big[\Big(\sum_{t=1}^n\mathbb I(c_t=m)- \sum_{t=1}^n\mathbb I(c_t=m)\mathbb I(X_t = x_m^*)\Big)\mathbb I(\cD)\Big]\\
&\geq&\Delta_{\min}  \mathbb E \Big[\frac{1}{2}\sum_{t=1}^n\mathbb I(c_t=m)\mathbb I(\cD)\Big].
\end{eqnarray*}
Define another event $\cB$ as follows,
\begin{equation}\label{def:event_B}
    \cB = \Big\{\frac{1}{2}\sum_{t=1}^n\mathbb I(c_t = m)\geq \frac{np_m}{2}-\delta/2\Big\},
\end{equation}
where $\delta>0$ will be chosen later and $p_m$ is the probability that the environment picks context $m$. From the definition of $c_t$, we have $\mathbb E[\sum_{t=1}^n\mathbb I(c_t=m)] = np_m.$
By the standard Hoeffding's inequality \citep{vershynin2010introduction}, it holds that
\begin{equation*}
    \mathbb P\Big(\frac{1}{2}\sum_{t=1}^n\mathbb I(c_t=m)- \frac{n p_m}{2}\geq-\frac{\delta}{2}\Big)\geq 1-\exp(-\frac{2\delta^2}{n}),
\end{equation*}
which implies
\begin{equation*}
    \mathbb P(\cB^c) \leq \exp(-2\delta^2/n).
\end{equation*}
By the definition of events $\cD,\cB$ in \eqref{def:event_D},\eqref{def:event_B}, we have
\begin{eqnarray*}
R_{\theta}^{\pi}(n)&\geq& \Delta_{\min} \mathbb E\Big[\frac{1}{2}\sum_{t=1}^n\mathbb I(c_t=m)\mathbb I(\cD)\mathbb I(\cB)\Big]\\
&\geq& \Delta_{\min} \mathbb E\Big[(\frac{1}{2}np_m-\frac{\delta}{2})\mathbb I(\cD)\mathbb I(\cB)\Big]\\
&=&\Delta_{\min} (\frac{1}{2}np_m-\frac{\delta}{2})\mathbb P(\cD\cap \cB)\\
&\geq& \Delta_{\min} (\frac{1}{2}np_m-\frac{\delta}{2})(\mathbb P(\cD)-\mathbb P(\cB^c)).
\end{eqnarray*}
Letting $\delta = np_m/2$, we have
\begin{eqnarray}\label{eqn:lower_bound_R}
R_{\theta}^{\pi}(n) \geq \Delta_{\min}\frac{np_m}{4}\Big(\mathbb P(\cD) - \exp(-\frac{np_m^2}{2})\Big).
\end{eqnarray}
On the other hand, we let $\tilde{\mathbb E}$ is taken with respect to probability measures $\tilde{\mathbb P}$. Then $R_{\tilde{\theta}}^{\pi}(n)$ can be lower bounded as follows,
\begin{eqnarray*}
R_{\tilde{\theta}}^{\pi}(n)
&=&\tilde{\mathbb E}\Big[\sum_{m=1}^M\sum_{t=1}^n\mathbb I(c_t=m)\tilde{\Delta}_{X_t}^{m}\Big]\\
&\geq& \tilde{\mathbb E}\Big[\sum_{t=1}^n \mathbb I(c_t=m)\mathbb I(X_t = x_m^*)\Big]\tilde{\Delta}_{x_m^*}^{m}, \\
\end{eqnarray*}
where we throw out all the sub-optimality gap terms except $\tilde{\Delta}_{x_m^*}^{m}$. Using the fact that $\tilde{\Delta}_{x_m^*}^{m}$ is $\varepsilon$-suboptimal, it holds that
\begin{eqnarray}\label{eqn:lower_bound_R'}
R_{\tilde{\theta}}^{\pi}(n)&\geq& \varepsilon \tilde{\mathbb E}\Big[(\sum_{t=1}^n \mathbb I(c_t=m)\mathbb I(X_t = x_m^*))\mathbb I(\cD^c)\Big]\nonumber\\
&>& \varepsilon \tilde{\mathbb E}\Big[\frac{1}{2}\sum_{t=1}^n \mathbb I(c_t=m)\mathbb I(\cD^c)\Big]\nonumber\\
&\geq& \varepsilon \tilde{\mathbb E}\Big[\frac{1}{2}\sum_{t=1}^n \mathbb I(c_t=m)\mathbb I(\cD^c)\mathbb I(\cB)\Big]\nonumber\\
&\geq& \varepsilon(\frac{np_m}{2}-\frac{\delta}{2})\tilde{\mathbb P}(\cD^c\cap \cB)\nonumber\\
&\geq& \varepsilon(\frac{np_m}{2}-\frac{\delta}{2}) (\tilde{\mathbb P}(\cD^c) - \tilde{\mathbb P}(\cB^c))\nonumber\\
&\geq& \varepsilon(\frac{np_m}{2}-\frac{\delta}{2}) (\tilde{\mathbb P}(\cD^c) - \exp(-\frac{2\delta^2}{n}))\nonumber\\
&=& \varepsilon\frac{np_m}{4}\tilde{\mathbb P}(\cD^c) - \varepsilon\frac{np_m}{4}\exp(-\frac{np_m^2}{2}).
\end{eqnarray}
Now we have derived the lower bounds \eqref{eqn:lower_bound_R}\eqref{eqn:lower_bound_R'} for $R_{\theta}^{\pi}(n),R_{\tilde{\theta}}^{\pi}(n)$ respectively.

\paragraph{Step Three.} Combining the lower bounds of $R_{\theta}^{\pi}(n)$ and $R_{\tilde{\theta}}^{\pi}(n)$ together, it holds that
\begin{equation*}
   R_{\theta}^{\pi}(n) +R_{\tilde{\theta}}^{\pi}(n)\geq \frac{np_m}{4}\Big(\mathbb P(\cD)\Delta_{\min}+\tilde{\mathbb P}(\cD^c)\varepsilon\Big)- \frac{np_m}{4}\exp(-\frac{np_m^2}{2})(\varepsilon+\Delta_{\min}).
\end{equation*}
Letting $\varepsilon\leq \Delta_{\min}$, we have
\begin{equation*}
     R_{\theta}^{\pi}(n) +R_{\tilde{\theta}}^{\pi}(n)\geq \varepsilon\frac{np_m}{4}\Big(\mathbb P(\cD)+\tilde{\mathbb P}(\cD^c)\Big)- \frac{np_m}{4}\exp(-\frac{np_m^2}{2})2\Delta_{\min}.
\end{equation*}
This implies 
\begin{equation}\label{eqn:bound_PD}
    \frac{ R_{\theta}^{\pi}(n) +R_{\tilde{\theta}}^{\pi}(n)}{\varepsilon np_m/4} + \frac{1}{\varepsilon}\exp(-\frac{np_m^2}{2})2\Delta_{\min}\geq \mathbb P(\cD) + \tilde{\mathbb P}(\cD^c).
\end{equation}
Plugging \eqref{eqn:bound_PD} into \eqref{eqn:lower_bound1}, we have
\begin{eqnarray*}
\frac{1}{2}\|\theta-\tilde{\theta}\|_{\bar{G}_n}^2&\geq& \log\Big(\frac{1}{2(\mathbb P(\cD)+\tilde{\mathbb P}(\cD^c))}\Big)\\
&\geq& \log \Big(\frac{1}{ \frac{R_{\theta}^{\pi}(n) +R_{\tilde{\theta}}^{\pi}(n)}{\varepsilon np_m/8} + \frac{1}{\varepsilon}\exp(-\frac{np_m^2}{2})4\Delta_{\min}}\Big)\\
&=& \log \Big(\frac{n}{ \frac{R_{\theta}^{\pi}(n) +R_{\tilde{\theta}}^{\pi}(n)}{\varepsilon p_m/8} + \frac{n}{\varepsilon}\exp(-\frac{np_m^2}{2})4\Delta_{\min}}\Big)\\
& = &\log(n) - \log \Big(\frac{R_{\theta}^{\pi}(n) +R_{\tilde{\theta}}^{\pi}(n)}{\varepsilon p_m/8} + \frac{4n}{\varepsilon}\exp(-\frac{np_m^2}{2})\Delta_{\min}\Big).
\end{eqnarray*}
Dividing by $\log (n)$ for both sides, we reach
\begin{eqnarray*}
\frac{\|\theta-\tilde{\theta}\|_{\bar{G}_n}^2}{2\log(n)} \geq 1- \frac{\log \Big(\frac{R_{\theta}^{\pi}(n) +R_{\tilde{\theta}}^{\pi}(n)}{\varepsilon p_m/8} + \frac{4n}{\varepsilon}\exp(-\frac{np_m^2}{2})\Delta_{\min}\Big)}{\log(n)}.
\end{eqnarray*}
From the definition of consistent policies \eqref{def:consis_policy}, it holds that
\begin{eqnarray*}
\limsup_{n\to \infty} \frac{\log(R_{\theta}^{\pi}(n) +R_{\tilde{\theta}}^{\pi}(n))}{\log(n)} \leq 0.
\end{eqnarray*}
In addition, by using the fact that $\lim_{n\to \infty}n\exp(-n)=0$, it follows that
\begin{equation}\label{eqn:liminf}
    \liminf_{n\to \infty} \frac{\|\theta-\tilde{\theta}\|_{\bar{G}_n}^2}{2\log(n)} \geq 1.
\end{equation}
\paragraph{Step Four.} Let's denote 
\begin{equation*}
    \rho_n(H)= \frac{\|x-x^*_m\|^2_{\bar{G}_n^{-1}}\|x-x^*_m\|_{H\bar{G}_n H}^2}{\|x-x^*_m\|_H^4}.
\end{equation*}
Then we can rewrite
\begin{eqnarray*}
\frac{1}{2}\|\theta-\tilde{\theta}\|_{\bar{G}_n}^2 = \frac{(\Delta_{x}^{m} + \varepsilon)^2}{2\|x-x^*_m\|^2_{\bar{G}_n^{-1}}}\rho_n(H).
\end{eqnarray*}
Plugging this into \eqref{eqn:liminf} and letting $\varepsilon$ to zero, we see that
\begin{equation}
    \liminf_{n\to\infty} \frac{\rho_n(H)}{\|x-x_m^*\|^2_{\bar{G}^{-1}_n}\log(n)}\geq \frac{2}{(\Delta_{x}^{m})^2}\,.
    \label{eq:lbc}
\end{equation}
Now, we consider the following lemma, extracted from the proof of Theorem~25.1 of the book by  \cite{lattimore2018bandit}. The detailed proof is deferred to Section \ref{proof:techlemma}.

\begin{lemma}\label{lem:techlemma}
Let $\{G_n\}_{n\ge 0}$ be a sequence of $d\times d$ positive definite matrices, $s\in \mR^d$.
For $H$ positive semi-definite $d\times d$ matrix such that $\|s\|_H>0$ and $n\ge 0$, let
$\rho_n(H)=\frac{\|s\|^2_{G_n^{-1}}\|s\|_{H G_n H}^2}{\|s\|_H^4}$.
Assume that 
$\liminf_{n\to\infty} \frac{\lambda_{\min}(G_n)}{\log(n)}>0$
and that
for some $c>0$,
\begin{align}
\liminf_{n\to\infty} \frac{\rho_n(H)}{\|s\|^2_{G_n^{-1}}\log(n)}\ge c\,.
\label{eq:lem:techc1}
\end{align}
Then, $\limsup_{n\to\infty} \log(n) \|s\|_{G_n^{-1}}^2 \le 1/c$.
\end{lemma}

The proof of $\liminf_{n\to\infty} \frac{\lambda_{\min}(G_n)}{\log(n)}>0$ could refer Appendix C in \cite{lattimore2017end}. Clearly, this lemma with $G_n=\bar{G}_n$, $c = 2/(\Delta_x^m)^2$, $H = \lim_{n\to\infty}\bar{G}_n^{-1}/\|\bar{G}_n^{-1}\|$ and $s=x-x^*_m$ gives the desired statement.

\hfill $\blacksquare$\\

\subsection{Proof of Theorem \ref{thm:upper_bound}: Asymptotic Upper Bound}\label{subsec:proof_upper_bound}

 We write $\Delta_{\max}=\max_{x, m}\Delta_x^m$ and abbreviate $R(n) = R_{\theta}^{\pi}(n)$. From the design of the initialisation, $G_t$ is guaranteed to be invertible since each $\cA^m$ is assumed to span $\mathbb R^d$. The regret during the initialisation is at most $d\Delta_{\max}\approx o(\log(n))$ and thus we ignore the regret during initialisation in the following.

First, we introduce a refined concentration inequality for the least square estimator constructed by adaptive data. The proof could refer to the proof of Theorem 8 in \cite{lattimore2017end}.
\begin{lemma}\label{lemma:concentr}
Suppose for $t\geq d$, $G_t$ is invertible. For any $\delta\in(0, 1)$, we have
\begin{equation*}
    \mathbb P\Big(\exists t\geq d, \exists x\in\cA, \text{such that} \ \big|\langle x,\hat{\theta}_t\rangle - \langle x,\theta\rangle\big|\geq \|x\|_{G_t^{-1}}f_{n,\delta}^{1/2}\Big)\leq \delta,
\end{equation*}
and
\begin{equation}\label{eqn:f_n_delta}
    f_{n,\delta} = 2\Big(1+\frac{1}{\log (n)}\Big)\log (1/\delta) + cd\log(d\log(n)),
\end{equation}
where $c>0$ is some universal constant. We write $f_{n} = f_{n,1/n}$ for short.
\end{lemma}
Let us define the event $\cB_t$ as follows
\begin{eqnarray}\label{eqn:events}
\cB_t = \Big\{ \exists t\geq d, \exists x\in\cA, \text{such that} \ |x^{\top}\hat{\theta}_t-x^{\top}\theta|\geq \|x\|_{G_t^{-1}}f_{n}^{1/2} \Big\}.
\end{eqnarray}
From Lemma \ref{lemma:concentr}, we have $\mathbb P(\cB_t)\leq 1/n$ by choosing $\delta = 1/n$. We decompose the cumulative regret with respect to event $\cB_t$ as follows,
\begin{eqnarray}\label{eqn:regret_decom1}
R(n) &=& \mathbb E\Big[\sum_{t=1}^n\sum_{x\in\cA_{-}^{c_t}}\Delta_x^{c_t}\mI(X_t=x)\Big]\nonumber\\
& = &\mathbb E\Big[\sum_{t=1}^n\sum_{x\in\cA_{-}^{c_t}}\Delta_x^{c_t}\mI(X_t=x, \cB_t)\Big]+ \mathbb E\Big[\sum_{t=1}^n\sum_{x\in\cA_{-}^{c_t}}\Delta_x^{c_t}\mI(X_t=x, \cB_t^c)\Big].
\end{eqnarray}

To bound the first term in \eqref{eqn:regret_decom1}, we observe that
\begin{eqnarray}\label{eqn:asy_bound}
&&\limsup_{n\to \infty}\frac{\mathbb E\Big[\sum_{t=1}^n\sum_{x\in\cA_{-}^{c_t}}\Delta_{x}^{c_t}\mI(X_t=x, \cB_t)\Big]}{\log (n)} \nonumber\\
&=& \limsup_{n\to \infty}\frac{\mathbb E\Big[\sum_{t=1}^n\Delta_{X_t}^{c_t}\mI(\cB_t)\Big]}{\log(n)}\leq \limsup_{n\to\infty}\frac{\Delta_{\max}\sum_{t=1}^n\mathbb P(\cB_t)}{\log(n)} = \limsup_{n\to\infty}\frac{\Delta_{\max}\sum_{t=1}^n\frac{1}{n}}{\log(n)}\nonumber\\ &=&\limsup_{n\to\infty}\frac{\Delta_{\max}}{\log(n)}=0.
\end{eqnarray}
To bound the second term in \eqref{eqn:regret_decom1}, we define the event $\cD_{t,c_t}$ as follows,
\begin{equation}\label{def:Dt}
    \cD_{t,c_t} = \left\{\forall x\in\cA^{c_t}, \|x\|_{G_t^{-1}}^2\leq \max\Big\{\frac{(\hat{\Delta}_{\min}(t))^2}{f_n}, \frac{(\Delta_x^{c_t}(t))^2}{f_n}\Big\}\right\}.
\end{equation}
When $\cD_{t,c_t}$ occurs, the algorithm exploits at round $t$. Otherwise, the algorithm explores at round $t$. We decompose the second term in \eqref{eqn:regret_decom1} as the exploitation regret and exploration regret:
\begin{eqnarray}\label{eqn:regret_decom2}
&&\mathbb E\Big[\sum_{t=1}^n\sum_{x\in\cA_{-}^{c_t}}\Delta_x^{c_t}\mI(X_t=x, \cB_t^c)\Big]\nonumber \\
&=& \mathbb E\Big[\sum_{t=1}^n\sum_{x\in\cA_{-}^{c_t}}\Delta_x^{c_t}\mI(X_t=x, \cB_t^c, \cD_{t,c_t})\Big]+ \mathbb E\Big[\sum_{t=1}^n\sum_{x\in\cA_{-}^{c_t}}\Delta_x^{c_t}\mI(X_t=x, \cB_t^c, \cD_{t,c_t}^c)\Big].
\end{eqnarray}
We bound those two terms in Lemmas \ref{lemma:exploitation_regret}-\ref{lemma:exploration_regret} respectively.

\begin{lemma}\label{lemma:exploitation_regret}
The exploitation regret satisfies
\begin{equation}\label{eqn:exploitation_regret}
    \limsup_{n\to\infty}\frac{\mathbb E\Big[\sum_{t=1}^n\sum_{x\in\cA_{-}^{c_t}}\Delta_x\mI(X_t=x, \cB_t^c, \cD_{t,c_t})\Big]}{\log(n)}=0
\end{equation}
\end{lemma}

\begin{lemma}\label{lemma:exploration_regret}
The exploration regret satisfies
\begin{equation}\label{eqn:exploration_regret}
    \limsup_{n\to\infty}\frac{\mathbb E\Big[\sum_{t=1}^n\sum_{x\in\cA_{-}^{c_t}}\Delta_x\mI(X_t=x, \cB_t^c, \cD_{t,c_t}^c)\Big]}{\log(n)}\leq \cC(\theta, \cA^1,\ldots, \cA^M),
\end{equation}
where $\cC(\theta, \cA^1,\ldots, \cA^M)$ is defined in Theorem \ref{thm:lower_bound}.
\end{lemma}
Combining Lemmas \ref{lemma:exploitation_regret}-\ref{lemma:exploration_regret} together, we reach our conclusion. \hfill $\blacksquare$\\

\section{Proofs of Several lemmas}\label{sec:proof:main_lemmas}

\subsection{Proof of Lemma \ref{lemma:exploitation_regret}: Exploitation Regret}\label{proof:exploitation_regret}

When $\cB_t^c$ defined in \eqref{eqn:events} occurs, we have
\begin{equation}\label{eqn:confid}
    \max_{x\in\cA}\big|\langle \hat{\theta}_t-\theta, x\rangle\big|\leq \|x\|_{G_t^{-1}}f_{n}^{1/2}.
\end{equation}
 When $\cD_{t,m}$ defined in \eqref{def:Dt} occurs, we have 
\begin{equation}\label{eqn:weighted_norm}
    \|x\|_{G_t^{-1}}^2\leq \max\Big\{\frac{\hat{\Delta}_{\min}^2(t)}{f_n}, \frac{(\hat{\Delta}_{x}^m(t))^2}{f_n}\Big\}=\frac{(\hat{\Delta}_{x}^m(t))^2}{f_n},
\end{equation}
holds for any action $x\in\cA^m$ and $\hat{\Delta}_{x}^m(t)>0$. If $x_m^* = \hat{x}_m^*(t)$, there is no regret occurred. Otherwise, putting \eqref{eqn:confid} and \eqref{eqn:weighted_norm} together with the optimal action $x_m^*$, it holds that
\begin{equation}\label{eqn:estimation_error}
    |\langle \hat{\theta}_t-\theta, x_m^*\rangle|\leq \|x_m^*\|_{G_t^{-1}}f_n^{1/2}\leq \hat{\Delta}_{x_{m}^*}^m(t).
\end{equation}
 We decompose the sub-optimality gap of $\hat{x}_m^*(t)$ as follows,
\begin{eqnarray}\label{eqn:E1}
&&\langle x_m^*,\theta\rangle - \langle \hat{x}_{m}^*(t), \theta\rangle\nonumber\\
&=&\langle x_m^*, \theta-\hat{\theta}_t\rangle + \langle x_{m}^*, \hat{\theta}_t\rangle -  \langle \hat{x}_{m}^*(t), \theta - \hat{\theta}_t\rangle - \langle \hat{x}_{m}^*(t), \hat{\theta}_t\rangle\nonumber\\
&=& \langle x_m^*, \theta-\hat{\theta}_t\rangle - \hat{\Delta}_{x_{m}^*}^m(t) + \langle \hat{x}_{m}^*(t), \hat{\theta}_t-\theta\rangle\nonumber\\
&\leq& \langle \hat{x}_{m}^*(t), \hat{\theta}_t-\theta\rangle.
\end{eqnarray}

 For each $x\in\cA$, we define 
\begin{equation}\label{def:tau_x}
\begin{split}
    \tau_x = \min\Big\{&N:\forall t\geq d, \cD_{t,c_t} \ \text{occurs}, N_x(t)\geq N, \text{implies} \ |\langle \hat{\theta}_t-\theta, x\rangle|\leq \frac{\Delta_{\min}}{2}\Big\}.
    \end{split}
\end{equation}
When $N_{\hat{x}_{m}^*(t)}(t)\geq \tau_{\hat{x}_{m}^*(t)}$, it holds that
\begin{equation*}
    |\langle \hat{\theta}_t-\theta, \hat{x}_{m}^*(t)\rangle|\leq \frac{\Delta_{\min}}{2}.
\end{equation*}
Together with \eqref{eqn:E1}, we have
\begin{equation*}
    \langle x_m^*,\theta\rangle - \langle \hat{x}_{m}^*(t), \theta\rangle \leq \frac{\Delta_{\min}}{2}.
\end{equation*}
Combining this with the fact that the instantaneous regret either vanishes or is larger than $\Delta_{\min}$, it indicates $x_{m}^* = \hat{x}_{m}^*(t)$. Therefore, we can decompose the exploitation regret with respect to event $\{N_{\hat{x}_{m}^*(t)}(t)\geq \tau_{\hat{x}_{m}^*(t)}\}$ as follows,
\begin{eqnarray}\label{eqn:R1n}
&&\mathbb E\Big[\sum_{t=1}^n\sum_{x\in\cA_{-}^{c_t}}\Delta_x^{c_t}\mI(X_t=x, \cB_t^c, \cD_{t,c_t})\Big]\nonumber\\
&\leq& \mathbb E\Big[\sum_{m=1}^M\sum_{t=1}^n\sum_{x\in\cA_{-}^m}\Delta_x^m\mI\Big(X_t=x, \cB_t^c,\cD_{t,m}, N_{\hat{x}_{m}^*(t)}(t)\geq \tau_{\hat{x}_{m}^*(t)}\Big)\Big]\nonumber\\
&+&\mathbb E\Big[\sum_{m=1}^M\sum_{t=1}^n\sum_{x\in\cA_{-}^m}\Delta_x^m\mI\Big(X_t=x, \cB_t^c,\cD_{t,m}, N_{\hat{x}_{m}^*(t)}(t)< \tau_{\hat{x}_{m}^*(t)}\Big)\Big].
\end{eqnarray}
 During exploiting the algorithm always executes the greedy action. When $x_m^* = \hat{x}_{m}^*(t)$ the first term in \eqref{eqn:R1n} results in no regret. For the second term in \eqref{eqn:R1n}, we have 
\begin{eqnarray}\label{eqn:second_term}
&&\mathbb E\Big[\sum_{m=1}^M\sum_{t=1}^n\sum_{x\in\cA_{-}^m}\Delta_x^m\mI\Big(X_t=x, \cB_t^c,\cD_{t,m}, N_{\hat{x}_{m}^*(t)}< \tau_{\hat{x}_{m}^*(t)}\Big)\Big]\nonumber\\
&\leq&  \mathbb E\Big[\sum_{m=1}^M\sum_{t=1}^n\mI\Big( \cB_t^c,\cD_{t,m},  N_{\hat{x}_{m}^*(t)}(t)< \tau_{\hat{x}_{m}^*(t)}\Big)\Big]\Delta_{\max}\nonumber\\
&\leq& \sum_{m=1}^M\sum_{x\in\cA}\mathbb E(\tau_x)\Delta_{\max}\leq \sum_{x\in\cA}\mathbb E[\tau_x]\Delta_{\max}.
\end{eqnarray}

It remains to bound $\mathbb E[\tau_x]$ for any $x\in\cA$. Let 
\begin{equation*}
    \Lambda = \min\Big\{\lambda\geq 1:\forall t\geq d, |\langle \hat{\theta}_t-\theta, x\rangle|\leq \|x\|_{G_t^{-1}}f_{n, 1/\lambda}^{1/2}\Big\}.
\end{equation*}
From the definition of $\tau_x$ in \eqref{def:tau_x}, we have
\begin{equation*}
    \tau_x\leq \max\Big\{N: (f_{n, 1/\lambda}/N)^{1/2}\geq \frac{\Delta_{\min}}{2}\Big\}, 
\end{equation*}
which implies $\tau_x \leq 4f_{n, 1/\Lambda}/\Delta_{\min}^2$.
From Lemma \ref{lemma:concentr}, we know that $\mathbb P(\Lambda\geq \lambda) \leq 1/\lambda$, which implies $\mathbb E[\log \Lambda]\leq 1$. Overall, 
\begin{equation}\label{eqn:bound_tau}
    \mathbb E[\tau_x]\leq \frac{4\mathbb E[f_{\Lambda}]}{\Delta_{\min}^2}\leq  \frac{8(1+1/\log(n)) + 4cd\log (d\log(n))}{\Delta_{\min}^2}.
\end{equation}
Combining \eqref{eqn:R1n}-\eqref{eqn:bound_tau} together, we reach 
\begin{equation}\label{eqn:eqn11}
\begin{split}
    &\limsup_{n\to\infty}\frac{\mathbb E\Big[\sum_{t=1}^n\sum_{x\in\cA_{-}^{c_t}}\Delta_x\mI(x_t=x, \cB_t^c, \cD_{t,c_t})\Big]}{\log(n)}\\
    \leq& \limsup_{n\to\infty} \frac{|\cA|\Delta_{\max}\big(8(1+1/\log(n)) + 4cd\log (d\log(n))\big)}{\Delta_{\min}^2 \log(n)} = 0.
\end{split}
\end{equation}
This ends the proof. \hfill $\blacksquare$\\

\subsection{Proof of Lemma \ref{lemma:exploration_regret}: Exploration Regret}\label{proof:exploration_regret}

If all the actions $x\in\cA$ satisfy 
\begin{equation}\label{eqn:N_bound}
N_x(t)\geq \min\Big\{f_n/\hat{\Delta}^2_{\min}(t), T_x(\hat{\Delta}(t))\Big\},
\end{equation}
the following holds using Lemma \ref{lemma:xH},
\begin{equation*}
    \|x\|_{G_t^{-1}}^2\leq \max \Big\{\frac{\hat{\Delta}_{\min}^2(t)}{f_n}, \frac{(\hat{\Delta}_x^{c_t}(t))^2}{f_n}\Big\}, \ \text{for any} \ x\in\cA.
\end{equation*}
 In other words, this implies if there exists an action $x$ such that \eqref{eqn:N_bound} does not hold, e.g. $\cD_{t,c_t}^c$ occurs, there must exist an action $x'\in\cA$ ($x$ and $x'$ may not be the identical) satisfying
$$
N_{x'}(t)\leq \min\Big\{f_t/\hat{\Delta}^2_{\min}(t),T_{x'}(\hat{\Delta}(t))\Big\}.
$$
Based on the criterion in Algorithm \ref{alg:main}, we should explore. However, if $x'$ does not belong to $\cA^{c_t}$ and all the actions within $\cA^{c_t}$ have been explored sufficiently according to the approximation optimal allocation, this exploration is interpreted as ``wasted''. To alleviate the regret of the wasted exploration, the algorithm acts optimistically as LinUCB.

Let's define a set that records the index of action sets that has not been fully explored until round $t$,
\begin{equation}
    \cM_t = \Big\{m:\exists x\in\cA^m, N_x(t)\leq \min\{f_n/\hat{\Delta}^2_{\min}(t),T_{x}(\hat{\Delta}(t)) \}\Big\}.
\end{equation}
When $\cD^c_{t,c_t}$ occurs, it means that $\cM_t\neq \emptyset$. If $\cD^c_{t,c_t}$ occurs but $c_t$ does not belong to $\cM_t$, the algorithm suffers a wasted exploration. We decompose the exploration regret according to the fact if $c_t $ belongs to $\cM_t$,
\begin{eqnarray}\label{eqn:regret_decom}
   &&\mathbb E\Big[\sum_{t=1}^n\sum_{x\in\cA_{-}^{c_t}}\Delta_x\mI(X_t=x, \cB_t^c, \cD_{t,c_t}^c\Big] \nonumber\\
   &=&  \underbrace{\mathbb E\Big[\sum_{t=1}^n\sum_{x\in\cA_{-}^{c_t}}\Delta_x\mI(X_t=x, \cB_t^c, \cD_{t,c_t}^c, c_t\in\cM_t)\Big]}_{R_{\text{ue}}:\text{unwasted exploration}}\nonumber\\
   &&+\underbrace{\mathbb E\Big[\sum_{t=1}^n\sum_{x\in\cA_{-}^{c_t}}\Delta_x\mI(X_t=x, \cB_t^c, \cD_{t,c_t}^c, c_t\notin\cM_t)\Big]}_{R_{\text{we}}:\text{wasted exploration}}.
\end{eqnarray}

We will bound the unwasted exploration regret and wasted exploration regret in the following two lemmas respectively.

\begin{lemma}\label{lemma:unwasted_regret}
The regret during the unwasted explorations satistifies
\begin{equation}\label{eqn:unwasted_regret}
    \limsup_{n\to\infty} \frac{R_{\text{ue}}}{\log(n)} \leq \cC(\theta, \cA_1,\ldots, \cA_M).
\end{equation}
\end{lemma}
The detailed proof is deferred to Section \ref{proof:unwasted_regret}.

\begin{lemma}\label{lemma:wasted_regret}
The regret during the wasted explorations satisfies
\begin{equation}\label{eqn:wasted_regret}
    \limsup_{n\to\infty} \frac{R_{\text{we}}}{\log(n)} = 0.
\end{equation}
\end{lemma}
The detailed proof is deferred to Section \ref{proof:wasted_regret}.

Putting \eqref{eqn:regret_decom}-\eqref{eqn:wasted_regret} together, we reach \begin{equation*}
      \limsup_{n\to\infty}\frac{\mathbb E\Big[\sum_{t=1}^n\sum_{x\in\cA_{-}^{c_t}}\Delta_x^{c_t}\mI(X_t=x, \cB_t^c, \cD_{t,c_t}^c)\Big]}{\log(n)}\leq \cC(\theta, \cA^1,\ldots, \cA^M),
\end{equation*}
which ends the proof.

\hfill $\blacksquare$\\

\subsection{Proof of Lemma \ref{lemma:unwasted_regret}: Unwasted Exploration}\label{proof:unwasted_regret}

First, we derive a lower bound for each $N_x(t)$ during the unwasted exploration. Denote $s(t)$ as the number of rounds for unwasted explorations until round $t$. Indeed, forced exploration can guarantee a lower bound for $N_x(t)$: $\min_{x\in\cA} N_x(t)\geq \varepsilon_t s(t)/2$. We prove this by the contradiction argument. Assume this is not true. There may exist $s(t)/2$ rounds $\{t_1, \ldots, t_{s(t)/2}\}\subset \{1,\ldots, t\}$ such that $\min_{x\in\cA}N_x(t)\leq \varepsilon_ts(t)$. After $|\cA|$ such rounds, we have $\min_x N_x(t)$ is incremented by at least 1 which implies $\min_x N_x(t)\geq s(t)/(2|\cA|)$. If $\varepsilon_t\leq 1/|\cA|$, it leads to the contradiction. This is satisfied when $t$ is large since $\varepsilon_t = 1/\log(\log t)$. 

Second, we set $\beta_n = 1/\log(\log(n))$ and define 
\begin{equation}\label{def:tau}
    \zeta = \min \Big\{s:\forall t\geq s, \forall x\in\cA, \text{such that} \ |\langle x, \hat{\theta}_t\rangle - \langle x, \theta \rangle|\leq \beta_n\Big\}.
\end{equation}
Then we decompose the regret during unwasted explorations with respect to event $\{s(t)\geq \zeta\}$ as follows,
\begin{eqnarray}\label{eqn:R_ue_decom}
R_{\text{ue}}&=&\mathbb E\Big[\sum_{t=1}^n\sum_{x\in\cA^{c_t}_{-}}\Delta_x\mI(X_t=x, \cB_t^c, \cD_{t,c_t}^c, c_t\in\cM_t)\Big]\nonumber\\
&=& \underbrace{\mathbb E\Big[\sum_{t=1}^n\sum_{x\in\cA^{c_t}_{-}}\Delta_x\mI(X_t=x, \cB_t^c, \cD_{t,c_t}^c, s(t)\geq \zeta, c_t\in\cM_t)\Big]}_{I_1}\nonumber\\
&&+ \underbrace{ \mathbb E\Big[\sum_{t=1}^n\sum_{x\in\cA^{c_t}_{-}}\Delta_x\mI(X_t=x, \cB_t^c, \cD_{t,c_t}^c, s(t)< \zeta, c_t\in\cM_t)\Big] }_{I_2}.
\end{eqnarray}
To bound $I_2$, we have 
\begin{eqnarray*}
    I_2 = \mathbb E\Big[\sum_{t=1}^n \Delta_{X_t}\mI(\cB_t^c, \cD_{t,c_t}^c,c_t\in\cM_t, s(t)<\zeta)\Big]\leq \Delta_{\max}\mathbb E\Big[\sum_{t=1}^n\mI(s(t)<\zeta, c_t\in\cM_t, \cD_{t,c_t}^c)\Big]\leq \Delta_{\max}\mathbb E[\zeta].
\end{eqnarray*}
It remains to bound $\mathbb E[\zeta]$.  Let's define
\begin{equation*}
    \Lambda = \min\Big\{\lambda:\forall t: \cD_{t,c_t}^c, \forall x\in\cA, s(t)\geq s,\text{such that} \ |\langle x, \hat{\theta}_t\rangle - \langle x, \theta \rangle|\leq \Big(\frac{2}{\varepsilon_t s(t)}f_{n, 1/\lambda}\Big)^{1/2}\Big\}.
\end{equation*}
From the definition of $\zeta$ in \eqref{def:tau}, we have
\begin{equation*}
    \zeta \leq \max\Big\{s:\Big(\frac{f_{n, 1/\lambda}}{\varepsilon_ts}\Big)^{1/2}\geq \beta_n\Big\},
\end{equation*}
which implies
\begin{equation}\label{eqn:upper_tau}
    \zeta \leq \frac{2f_{n, 1/\Lambda}}{\varepsilon_t\beta_n^2}.
\end{equation}
In addition, we define
\begin{equation*}
\Lambda' = \min\Big\{\lambda:\forall t\geq d, \forall x\in\cA, \text{such that} \ |\langle x, \hat{\theta}_t\rangle - \langle x, \theta \rangle|\leq \|x\|_{G_t^{-1}}f_{n, 1/\lambda}^{1/2}\Big\}.
\end{equation*}
Using the lower bound of $N_x(t)$, it holds that 
$$
\|x\|^2_{G^{-1}_t}\leq \frac{1}{N_x(t)}\leq \frac{2}{\varepsilon_t s(t)}.
$$
By Lemma \ref{lemma:concentr}, we have
\begin{equation*}
    \mathbb P\Big(\Lambda\geq \frac{1}{\delta}\Big)\leq \mathbb P\Big(\Lambda'\geq \frac{1}{\delta}\Big)\leq \delta,
\end{equation*}
which implies that $\mathbb E[\log \Lambda]\leq 1$. From \eqref{eqn:upper_tau},
\begin{equation}\label{eqn:exp_zeta}
    \mathbb E[\zeta] \leq \frac{2(1+1/\log(n)) + cd\log(\log (d\log(n)))}{\varepsilon_n\beta_n^2}.
\end{equation}

From \eqref{eqn:exp_zeta}, we have
\begin{equation}\label{eqn:bound_I2}
    \limsup_{n\to\infty}\frac{I_2}{\log(n)}\leq \limsup_{n\to\infty}\frac{\Delta_{\max}\mathbb E[\zeta]}{\log(n)}=0,
\end{equation}
since $\beta_n$ and $\varepsilon_n$ are both sub-logarithmic.
It remains to bound $I_1$. 
 When $s(t)\geq \zeta$, from the definition of $\zeta$ in \eqref{def:tau} we have
\begin{equation*}
    \langle x,\hat{\theta}_t\rangle - \langle x, \theta\rangle \leq \beta_n,
\end{equation*}
holds for any $x\in\cA$. For each $m\in[M]$, we have 
\begin{eqnarray*}
\hat{\Delta}_{x_m^*}(t) &=& \langle \hat{\theta}_t, \hat{x}_{m}^*(t)\rangle - \langle \hat{\theta}_t,x_m^*\rangle\\
&=& \langle \hat{\theta}_t, \hat{x}_{m}^*(t)\rangle- \langle \theta, \hat{x}_{m}^*(t) \rangle- \langle \hat{\theta}_t, x_m^*\rangle + \langle \theta, x_{m}^*\rangle -\langle \theta, x_m^*\rangle + \langle \theta, \hat{x}_{m}^*(t)\rangle\\
&\leq& 2\beta_t - \Delta_{\min}.
\end{eqnarray*}
When $n$ is sufficiently large, it holds that $\beta_n\leq \Delta_{\min}/2$. This implies $\hat{\Delta}_{x_m^*}(t) = 0$ such that $x_m^* = \hat{x}_{m}^*(t)$ for all $t:s(t)>\zeta$. For notation simplicity, we denote $\cE_t = \cB_t^c\cap \cD_{t,c_t}^c\cap \{s(t)\geq \zeta\}\cap \{c_t\in\cM_t\}$. When $\cE_t$ occurs, the algorithm is in the unwasted exploration stage and $x_m^* = \hat{x}_{m}^*(n)$. 

 When $\cD_{t,c_t}^c$ occurs and $c_t\in\cM_t$, there exists $x'\in\cA^{c_t}$ such that $N_{x'}(t)\leq \min(f_n/\hat{\Delta}_{\min}^2(t), T_{x'}(\hat{\Delta}(t)))$. From the design of Algorithm \ref{alg:main}, it holds that 
\begin{itemize}
\item If $x=b_1$, then $N_x(t)\leq \min(f_n/\hat{\Delta}_{\min}^2(t), T_x(\hat{\Delta}(t)))$.
    \item If $x=b_2$, then $N_x(t)=\min_{x\in\cA^{c_t}}N_x(t)\leq \min(f_n/\hat{\Delta}_{\min}^2(t), T_{x'}(\hat{\Delta}(t)))$.
\end{itemize}

Since the algorithm either pulls $b_1$ or $b_2$ in the unwasted exploration, it implies an upper bound for $s(t)$:
\begin{equation}\label{eqn:upper_s}
 s(t)\leq \sum_{x\in\cA^{c_t}}N_x(t)\leq |\cA|\max_{x}\min(f_n/\hat{\Delta}_{\min}^2(t), T_x(\hat{\Delta}(t))).  
\end{equation}

Let $\Lambda$ be the random variable given by
\begin{align*}
\Lambda = \min\left\{\lambda : \max_{x \in \cA} |\ip{x, \hat \theta_t - \theta}| \leq \|x\|_{G_t^{-1}} f_{n,1/\lambda}^{1/2} \text{ for all } t \in [n]\right\},
\end{align*}
where $f_{n,1/\lambda}$ is defined in \cref{eqn:f_n_delta}.
By the concentration inequality Lemma \ref{lemma:concentr}, for any $\lambda \geq 1$, 
\begin{align}
\Prob{\Lambda \geq \lambda} \leq 1/\lambda\,.
\end{align}
Hence the event $F = \{\Lambda \geq n\}$ satisfies $\Prob{F} \leq 1/n$. Denote $\alpha_{x}^m(\Delta) = T_x^m(\Delta)/f_n$ where $T_x^m(\Delta)$ 
is the solution of optimisation problem in Definition \ref{def:optimi} with true $\Delta$. 
Given $\upsilon > 0$ let 
\begin{align*}
\upsilon(\delta)
= \sup\left\{\|\alpha(\Delta) - \alpha(\tilde \Delta)\|_{\infty} : \|\tilde \Delta - \Delta\|_\infty \leq \delta\right\},
\end{align*}
where $\alpha(\Delta) = \{\alpha_x^m(\Delta)\}_{x\in\cA^m, m\in[M]}$.
By continuity assumption of $\alpha$ at $\Delta$ we have $\lim_{\delta\to 0} \upsilon(\delta) = 0$. Moreover, let's define
\begin{align*}
\tau_\delta = \min\left\{t : \max_{x \in \cA} |\ip{x, \hat \theta_s - \theta}| \leq \delta/2 \text{ for all } x \in \cA \text{ and } s \geq t\right\}\,. 
\end{align*}
Since $N_x(t) \geq \epsilon_n s(t) / 2$,
\begin{align*}
\max_{x \in \cA} |\ip{x, \hat \theta_t - \theta}| \leq \sqrt{\frac{2 f_{n,\Lambda}}{\epsilon_n s(t)}}\,.
\end{align*}
Therefore the number of exploration steps at time $\tau_\delta$ is bounded by $s(\tau_\delta) \leq 8 f_{n,1/\Lambda} \epsilon_n^{-1} \delta^{-2}$.

Let $(\delta_n)_{n=1}^\infty$ be a sequence with $\lim_{n\to\infty} \delta_n = 0$ and $\log(\log(n)) / \delta_n^2 = o(\log(n))$.
$I_{11}$ decomposed as
\begin{eqnarray}
I_{11} &=& \mathbb E\Big[\sum_{t=1}^n\sum_{x\in \cA^{c_t}_{-}}\Delta_x\mI(X_t =x, \cE_t)\Big]\nonumber\\
&\leq& \E\left[s(\tau_{\delta_n})\right] + \mathbb E\Big[\sum_{t=\tau_{\delta_n}}^n\sum_{x\in \cA^{c_t}_{-}}\Delta_x\mI(X_t =x, \cE_t)\Big].
\label{eq:explore1}
\end{eqnarray}
The first term in \eqref{eq:explore1} is bounded by
\begin{align*}
\E\left[s(\tau_{\delta_n})\right] \leq \frac{8}{\epsilon_n \delta_n^2} \E[f_{n,1/\Lambda}] = o(\log(n))\,,
\end{align*}
where we used the assumption on $(\delta_n)$ and the fact that $\E[f_{n,1/\Lambda}] = O(\log \log(n))$. By the continuity assumption, the following statement holds
\begin{eqnarray}\label{eqn:upper_N}
   \sum_{t=\tau_{\delta_n}+1}^n\mI(X_t = x, \cE_t)&\leq& \varepsilon_ns(n)+ f_n\min\Big(1/\hat{\Delta}_{\min}^2(n), \alpha_x^{c_t}(\hat{\Delta}(n))/2\Big)\nonumber\\
   &\leq& \varepsilon_ns(n)+ f_n\min\Big(\frac{1}{\hat{\Delta}_{\min}^2(n)}, (\alpha^{c_t}_x(\Delta) + \upsilon(\delta_n))/2\Big).
\end{eqnarray}

The second term in \eqref{eq:explore1} is bounded by
\begin{eqnarray*}
&&\mathbb E\Big[\sum_{t=\tau_{\delta_n}}^n\sum_{x\in \cA^{c_t}_{-}}\Delta_x\mI(X_t =x, \cE_t)\Big]\\
 &\leq& \mathbb E\Big[\sum_{m=1}^M\sum_{x\in \cA^{m}_{-}}\Delta_x\sum_{t=1}^n\mI(X_t =x, \cE_t)\Big]\\
&\leq&\mathbb E\Big[\sum_{m=1}^M\sum_{x\in \cA^{m}_{-}}\Delta_x\varepsilon_ns(n)\mI(\cE_n)\Big]+ \mathbb E\Big[\sum_{m=1}^M\sum_{x\in \cA^{m}_{-}}\Delta_x f_n (\alpha^{m}_x(\Delta) + \upsilon(\delta_n))/2\mI(\cE_n)\Big].
\end{eqnarray*}
To bound the second term, we take the limit as $n$ tends to infinity and the fact that $\lim_{n\to\infty} \upsilon(\delta_n) = 0$ and $f_n \sim 2 \log(n)$ shows that
\begin{eqnarray}
    \limsup_{n\to\infty}\frac{1}{\log(n)}\mathbb E\Big[\sum_{m=1}^M\sum_{x\in \cA^{m}_{-}}\Delta_x f_n (\alpha^{m}_x(\Delta) + \upsilon(\delta_n))/2\mI(\cE_n)\Big]\leq \cC(\theta, \cA^1, \ldots, \cA^M).
\end{eqnarray}

 We bound the first term in the following lemma. The detailed proofs are deferred to Section \ref{proof:forced_exploration}.
\begin{lemma}\label{lemma:forced_exploration}
The regret contributed by the forced exploration satisfies
\begin{equation*}
    \limsup_{n\to\infty} \frac{\mathbb E\Big[\sum_{x\in\cA^{c_t}_{-}}\Delta_x\varepsilon_ns(n)\mI(\cE_n)\Big]}{\log(n)} = 0.
\end{equation*}
\end{lemma}

This ends the proof. \hfill $\blacksquare$\\

\subsection{Proof of Lemma \ref{lemma:forced_exploration}: Forced Exploration Regret}\label{proof:forced_exploration}

By the upper bound of unwasted exploration counter $s(n)$ in \eqref{eqn:upper_s}, it holds that
\begin{eqnarray*}
\sum_{m=1}^M\sum_{x\in\cA^m_{-}}\Delta_x^m\varepsilon_ns(n) \mI(\cE_n)&\leq&\sum_{m=1}^M\sum_{x\in\cA^m_{-}}\Delta_x^m\varepsilon_n|\cA|\max_{x}\min(f_n/\hat{\Delta}_{\min}^2(n), T_x(\hat{\Delta}(n)))\mI(\cE_n)\\
&\leq & \varepsilon_n|\cA|\sum_{m=1}^M\sum_{x\in\cA^m_{-}}\Delta_x^mf_n/\hat{\Delta}_{\min}(n)\mI(\cE_n).
\end{eqnarray*}
When event $\cE_n$ occurs, 
\begin{eqnarray*}
    \max_{x\neq \hat{x}_{m}^*(n)}\frac{(\Delta_x^{m})^2}{(\hat{\Delta}_{x}(n))^2}&\leq& \max_{x\neq \hat{x}_{m}^*(n)} \frac{(\Delta_x^{m})^2}{(\Delta_x^m - 2\beta_n)^2}\nonumber\\
    &=& \max_{x\neq \hat{x}_{m}^*(n)}\Big(1+\frac{4(\Delta_x^m-\beta_n)\beta_n}{(\Delta_x^m-2\beta_n)^2}\Big)\leq 1+\frac{16\beta_n}{\Delta_{\min}},
\end{eqnarray*}
For any $x\in\cA^m$,
\begin{equation}\label{eqn:lower_bound_Delta_min}
    \hat{\Delta}_{\min}(n)\geq\frac{1}{1+16\beta_n/\Delta_{\min}}\Delta_{\min}.
\end{equation}
Since $\varepsilon_n=1/(\log \log(n))$, we have 
\begin{equation}\label{eqn:bound_I111}
    \limsup_{n\to\infty}\frac{\sum_{x\in\cA_{-}}\Delta_x\varepsilon_ns(n) \mI(\cE)}{\log(n)} = 0.
\end{equation}

This ends the proof. \hfill $\blacksquare$\\

\subsection{Proof of Lemma \ref{lemma:wasted_regret}: Wasted Exploration}\label{proof:wasted_regret}

First, we define 
\begin{equation}
    \cF_s = \Big\{\exists t\geq d, \exists x: \langle x, \hat{\theta}_t\rangle - \langle x,\theta \rangle \geq \|x\|_{G_t^{-1}}f_{n, 1/s^2}^{1/2}\Big\},
\end{equation}
where $f_{n, 1/s^2}$ is defined in Lemma \ref{lemma:concentr}. From Lemma \ref{lemma:concentr}, we also have $\mathbb P(\cF_s)\leq 1/s^2$. Let $s'(t), s(t)$ be the number of rounds for wasted explorations, unwasted explorations until round $t$ accordingly, and $x_t^*$ is the optimal arm at round $t$. We decompose the regret as follows
\begin{equation}
    R_{\text{we}} \leq \underbrace{\mathbb E\Big[\sum_{t\in\text{wasted}} \mI(\cF_{s'(t)})\Delta_{\max}\Big]}_{I_1}+ \underbrace{\mathbb E\Big[\sum_{t\in\text{unwasted}} \mI(\cF_{s'(t)}^c)\langle x_t^*-X_t, \theta\rangle\Big]}_{I_2}.
\end{equation}
To bound $I_1$, we have
\begin{equation}\label{bound_I1}
    I_1\leq \sum_{s=1}^n \mathbb P(\cF_s)\Delta_{\max} \leq \sum_{s=1}^n\frac{1}{s^2}\Delta_{\max} = (2-\frac{1}{n})\Delta_{\max}.
\end{equation}
To bound $I_2$, let's denote $\tilde{\theta}_t$ as the optimistic estimator. Following the standard one step regret decomposition (See the proof of Theorem 19.2 in \cite{lattimore2018bandit} for details), it holds that 
\begin{eqnarray*}
    \langle x_t^*- X_t, \theta\rangle 
    &=& \langle x_t^*, \theta\rangle - \langle X_t, \theta\rangle\\
    &\leq& \langle X_t, \tilde{\theta}_t\rangle - \langle X_t,\theta\rangle\\
    &=& \langle X_t, \hat{\theta}_t-\theta\rangle + \langle  X_t,\tilde{\theta}_t -\hat{\theta}_t\rangle.
\end{eqnarray*}
When $\cF_{s'(t)}^c$ occurs, we have
\begin{eqnarray*}
    \langle X_t, \hat{\theta}_t-\theta\rangle\leq \|X_t\|_{G_t^{-1}} f_{n, 1/(s'(t)^2)}, \langle X_t, \tilde{\theta}_t-\hat{\theta}_t\rangle\leq \|X_t\|_{G_t^{-1}} f_{n, 1/(s'(t)^2)}.
\end{eqnarray*}
Putting the above results together, we have 
\begin{equation*}
    \langle x_t^*- X_t, \theta\rangle \leq 2\|X_t\|_{G_t^{-1}}f^{1/2}_{n,1/(s'(t)^2)}.
\end{equation*}
Applying Lemma \ref{lemma:sum_Xt}, we can bound $I_2$ as follows
\begin{eqnarray}\label{bound_I2}
    I_2&\leq& \mathbb E\Big[2f_{n,1/(s'(t)^2)}^{1/2}\sum_{t\in\text{wasted}}\|X_t\|_{G_t^{-1}}\Big]\nonumber\\
    &\leq& \mathbb E\Big[2f_{n,1/(s'(t)^2)}^{1/2}\sqrt{2s'(n)d\log \Big(\frac{s'(n)+d}{d}\Big)}\Big].
\end{eqnarray}

Recall that $p_{\min} = \min_{m}p_m$ be the minimum probability that each action set arrives. It is easy to see $\mathbb P (c_t\in\cM_t|\cD_{t,c_t}^c) =\mathbb P(c_t\in\cM_t|\cM_t\neq \emptyset)= \sum_{m\in\cM_t} p_m\geq p_{\min}$.
We bound $s'(n)$ by $s(n)$ as follows
\begin{eqnarray}\label{bound_sn}
    &&\mathbb E[s'(n)] = \mathbb E\Big[\sum_{t=1}^n\mI\Big(\cD_{t,c_t}^c, c_t\notin \cM_t\Big)\Big] = \sum_{t=1}^n \mathbb P(\cD_{t,c_t}^c)\mathbb P(c_t\notin\cM_t|\cD_{t,c_t}^c)\nonumber\\
    &\leq& \frac{1}{p_{\min}}\sum_{t=1}^n \mathbb P(\cD_{t,c_t}^c)\mathbb P(c_t\in\cM_t|\cD_{t,c_t}^c)\nonumber\\
    &=&\frac{1}{p_{\min}} \mathbb E \Big[\sum_{t=1}^n\mI\Big(\cD_{t,c_t}^c,c_t\in\cM_t\Big)\Big] = \frac{1}{p_{\min}}\mathbb E[s(n)].
\end{eqnarray}
Putting \eqref{bound_I1}-\eqref{bound_sn} together, The regret in the wasted exploration can be upper bounded by
\begin{equation}\label{eqn:R_we_bound}
    R_{\text{we}} \leq (2-\frac{1}{n})\Delta_{\max} + \frac{2}{p_{\min}}\sqrt{2d\log \Big(\frac{s(n)/p_{\min}+d}{d}\Big)f_{n,(p_{\min}/s(n))^2}s(n)/p_{\min}},
\end{equation}
where $f_{n,(p_{\min}/s(n))^2}$ is defined in \eqref{lemma:concentr}.

Next, we recall the upper bound \eqref{eqn:upper_s} for the number of pulls in unwasted exploration,
\begin{eqnarray*}
    s(n)&\leq& |\cA|\max_x\min\Big\{f_n/\hat{\Delta}_{\min}(n), T_x(\hat{\Delta}(n))\Big\}\\
    &\leq& |\cA|f_n/\hat{\Delta}_{\min}(n).
\end{eqnarray*}
From \eqref{eqn:lower_bound_Delta_min}, we have
\begin{eqnarray*}
    \hat{\Delta}_{\min}(n)\geq\frac{1}{1+\delta_n}\Delta_{\min}\geq \frac{\Delta_{\min}^2}{\Delta_{\min}+16\beta_n},
\end{eqnarray*}
where $\beta_n = 1/\log (\log(n))$. Overall, we see $s(n)\leq \cO(\log(n))$. Plugging this into \eqref{eqn:R_we_bound}, we reach
\begin{equation*}
    \limsup_{n\to\infty} \frac{R_{\text{we}}}{\log(n)} = 0.
\end{equation*}
This ends the proof.  \hfill $\blacksquare$\\

\subsection{Proof of Lemma \ref{lem:techlemma}}\label{proof:techlemma}

First, we start by the following claim:
\begin{claim}\label{claim:mxclaim}
Assume $H_n$ is a sequence of $d\times d$ positive definite matrices such that $H_n \to H$ and $H$ is positive semidefinite. Then, $H H_n^{-1} H \to H$ as $n\to\infty$.
\end{claim}
\begin{proof}
Without loss of generality, we can assume that $H$ is given in the block matrix form
\begin{align*}
H = \begin{pmatrix}
A & 0 \\
0  & 0
\end{pmatrix}
\end{align*}
where $A$ is a nonsingular $m\times m$ matrix with $m>0$.
(If $m=0$, $H$ is the all zero matrix and the claim trivially holds.)
Consider the same block partitioning of $H_n$:
\begin{align*}
H_n = \begin{pmatrix}
A_n &  B_n \\
B_n^\top & D_n 
\end{pmatrix}\,,
\end{align*}
where $A_n$ is thus also an $m\times m$ matrix. 
Clearly, $A = \lim_{n\to\infty} A_n$ and $A_n$ is nonsingular (or $H_n$ would be singular),
while $B_n \to B$ and $D_n \to D$ where all entries in $B$ and $D$ are zero.
Then, as is well known,
\begin{align*}
H_n^{-1} = \begin{pmatrix}
A_n^{-1}+A_n^{-1} B_n S_n^{-1} B_n^\top A_n^{-1} & - A_n^{-1} B_n S_n^{-1} \\
- S_n^{-1} B_n^\top A_n^{-1} & S_n^{-1} 
\end{pmatrix}\,.
\end{align*}
where $S_n = D_n - B_n^\top A_n^{-1} B_n$ is the Schur-complement of block $D_n$ of matrix $H_n$.
Note that 
\begin{align*}
H H_n^{-1} H = 
\begin{pmatrix}
A(A_n^{-1}+A_n^{-1} B_n S_n^{-1} B_n^\top A_n^{-1})A & 0 \\
0 & 0
\end{pmatrix}\,.
\end{align*}
Since the matrix inverse is continuous if the limit is nonsingular, $A_n^{-1} \to A^{-1}$.
Clearly, it suffices to show that $A_n^{-1}+A_n^{-1} B_n S_n^{-1} B_n^\top A_n^{-1}  \to A^{-1}$.
Hence, it remains to check that $A_n^{-1} B_n S_n^{-1} B_n^\top A_n^{-1} \to 0$.
This follows because $B_n \to B$ and $D_n\to D$ and $S_n \to D - B^\top A^{-1} B=0$ where $D=0$ and $B=0$.
\end{proof}
\begin{proof}[Proof of \cref{lem:techlemma}]
Let  $L=\limsup_{n\to\infty} \log(n) \|s\|_{G_n^{-1}}^2$. We need to prove that $L\le 1/c$. 
Without loss of generality, assume that $L>0$ (otherwise there is nothing to be proven)
and that for some $H$ positive semidefinite matrix, $\zeta\in \RR$ and $\kappa\in \RR \cup \{\infty\}$,
{\em (i)}   $\log(n) \|s\|_{G_n^{-1}}^2 \to L$;
{\em (ii)}  $H_n=G_n^{-1}/\norm{G_n^{-1}} \to H$;
{\em (iii)} $\lambda_{\min}(G_n)/\log(n) \to \zeta>0$ and
{\em (iv)} $\frac{\rho_n(H)}{\log(n) \|s\|_{G_n^{-1}}^2} \to \kappa \ge c$.
We claim that $\|s\|_H>0$, hence $\rho_n(H)$ is well-defined  and in particular $\rho_n(H) \to 1$ as $n\to\infty$.
If this was true, then the proof was ready since
\begin{align*}
L = \lim_{n\to\infty} \frac{\log(n) \|s\|_{G_n^{-1}}^2}{\rho_n(H)} 
= \frac{1}{ \lim_{n\to\infty} \frac{\rho_n(H)}{\log(n) \|s\|_{G_n^{-1}}^2} } =1/\kappa \le 1/c\,.
\end{align*}

Hence, it remains to show the said claim.
We start by showing that $\|s\|_H>0$. For this note that $\norm{G_n^{-1}} = 1/\lambda_{\min}(G_n)$ and hence
\begin{align*}
\|s\|^2_{\frac{G_n^{-1}}{\norm{G_n^{-1}}}}
=
\frac{\lambda_{\min}(G_n)}{\log(n)}\,  \|s\|^2_{G_n^{-1}} \log(n)\,.
\end{align*}
Taking the limit of both sides, we get $\|s\|_H^2  \to \zeta L >0$.
Now,
\begin{align*}
\rho_n(H) 
= \frac{\|s\|^2_{G_n^{-1}}\|s\|^2_{H G_n H}}{\|s\|_H^4}
= \frac{\|s\|^2_{H_n}\|s\|^2_{H H_n^{-1} H}}{\|s\|_H^4}
\stackrel{n\to\infty}{\to}
\frac{\|s\|^2_{H}\|s\|_{H}^2}{\|s\|_H^4} = 1\,,
\end{align*}
where we used \cref{claim:mxclaim}.
\end{proof}

\subsection{Proof of \cref{thm:ucb}} \label{sec:thm:ucb}

Suppose that $\{x_m^* : m \in [M]\}$ spans $\mathbb R^d$.
Recall that LinUCB chooses 
\begin{align*}
X_t = \argmax_{x \in \cA^{c_t}} \ip{x, \hat \theta_{t-1}} + \norm{x}_{G_{t-1}^{-1}} \beta_t^{1/2}\,,
\end{align*}
where $\beta_t = O(d \log(t))$ is chosen so that
\begin{align*}
\Prob{\norm{\hat \theta_t - \theta}_{G_t} \geq \beta_t} \leq 1/t^3\,,
\end{align*}
which is known to be possible \citep[\S20]{lattimore2018bandit}.
Define $F_t$ to be the event that $\norm{\hat \theta_t - \theta}_{G_t} \geq \beta_t$.
Then the instantaneous pseudo-regret of LinUCB is bounded by
\begin{align*}
\Delta_t \leq \one_{F_t} + \ip{x^*_t - X_t, \theta} \leq \one_{F_t} + 2 \beta_t^{1/2} \norm{X_t}_{G_t^{-1}} \leq \one_{F_t} + 2 \sqrt{\beta_t \norm{G_t^{-1}}}\,,
\end{align*}
where the matrix norm is the operator name (in this case, maximum eigenvalue).
Let $\tau = 1 + \max\{t : F_t \text{ holds} \}$, which satisfies $\E[\tau] = O(1)$.
The cumulative regret after $\tau$ is bounded almost surely by
\begin{align*}
\sum_{t=\tau}^n \ip{x^*_t - X_t, \theta} = O\left(\sqrt{n} \log(n)\right)\,,
\end{align*}
where the Big-Oh hides constants that only depend on the dimension. Hence all optimal arms are played linearly often after $\tau$, which by the assumption that
$\{x_m^* : m \in [M]\}$ spans $\mathbb R^d$ implies that $\norm{G_t^{-1}} = O(1/t)$.
Hence the instantaneous regret for times $t \geq \tau$ satisfies
\begin{align*}
\Delta_t = O\left(\sqrt{\frac{\beta_t}{t}}\right)\,.
\end{align*}
Since $\Delta_t \in \{0\} \cup [\Delta_{\min}, 1]$, it follows that the regret vanishes once $\Delta_t < \Delta_{\min}$.
But by the previous argument and the assumption on $\beta_t$ we have for $t \geq \tau$ that 
\begin{align*}
\Delta_t \leq 2\sqrt{\beta_t \norm{G_t^{-1}}} = O\left(\sqrt{\frac{\log(t)}{t}}\right)\,.
\end{align*}
Hence for sufficiently large $t$ (independent of $n$) the regret vanishes, which completes the proof.

\section{Supporting Lemmas}\label{sec:supporting_lemmas}

\begin{lemma}[Bretagnolle-Huber Inequality]\label{lem:kl}
Let $\mathbb P$ and $\tilde{\mathbb P}$ be two probability measures on the same measurable space $(\Omega,\cF)$. Then for any event $\cD\in \cF$,
\begin{equation}\label{eqn:kl}
\mathbb P(\cD) + \tilde{\mathbb P}(\cD^c) \geq \frac{1}{2} \exp\left(-\text{KL}(\mathbb P, \tilde{\mathbb P})\right)\,,
\end{equation}
where $\cD^c$ is the complement event of $\cD$ ($\cD^c = \Omega\setminus \cD$) and $\text{KL}(\mathbb P, \tilde{\mathbb P})$ is the KL-divergence between $\PP$ and $\tilde{\mathbb P}$, which is defined as $+\infty$, if $\PP$ is not absolutely continuous with respect to $\tilde{\mathbb P}$, and is $\int_\Omega d\PP(\omega) \log \frac{d\PP}{d\tilde{\mathbb P}}(\omega)$ otherwise.
\end{lemma}
The proof can be found in the book of \citet{Tsybakov:2008:INE:1522486}. When $\text{KL}(\mathbb P, \tilde{\mathbb P})$ is small, we may expect the probability measure $\mathbb P$ is close to the probability measure $\tilde{\mathbb P}$. Note that $\mathbb P(\cD) + \mathbb P(\cD^c)=1$. If $\tilde{\mathbb P}$ is close to $\mathbb P$, we may expect $\mathbb P(\cD)+\tilde{\mathbb P}(\cD^c)$ to be large.

\begin{lemma}[Divergence Decomposition]\label{lem:inf-processing}
Let $\mathbb P$ and $\tilde{\mathbb P}$ 
be two probability measures on the sequence  $(A_1, Y_1,\ldots,A_n,Y_n)$ for a fixed
bandit policy $\pi$ interacting with a linear contextual bandit with standard Gaussian noise and parameters $\theta$ and $\tilde{\theta}$ respectively. Then the KL divergence of $\mathbb P$ and $\tilde{\mathbb P}$ can be computed  exactly and is given by
\begin{equation}\label{eqn:inf-processing}
\text{KL}(\mathbb P, \tilde{\mathbb P}) = \frac12 \sum_{x \in \mathcal A} \mathbb E[T_x(n)]\, \langle x, \theta - \tilde{\theta}\rangle^2\,,
\end{equation}
where $\mathbb E$ is the expectation operator induced by $\PP$. 
\end{lemma}
This lemma appeared as Lemma 15.1 in the book of \citet{lattimore2018bandit}, where the reader can also find the proof.

\begin{lemma}\label{lemma:sum_Xt}
Let $\{X_t\}_{t=1}^{\infty}$ be a sequence in $\mathbb R^d$ satisfying $\|X_t\|_2\leq 1$ and $G_t =\sum_{s=1}^tX_tX_t^{\top}$. Suppose that $\lambda_{\min}(G_d)\geq c$ for some strictly positive $c$. For all $n>0$, it holds that 
\begin{equation*}
    \sum_{t=d+1}^n\|X_t\|_{G_t^{-1}}\leq \sqrt{2nd\log(\frac{d+n}{d})}.
\end{equation*}
\end{lemma}

\begin{lemma}\label{lemma:xH}
Let $\varepsilon>0$ and denote $T(\hat{\Delta}(n))\in\mathbb R^{|\cA|}$ as the solution of the optimisation problem defined in  Definition \ref{def:optimi}. Then we define 
\begin{equation*}
    S_{\varepsilon}(\hat{\Delta}(n)) = \min\Big\{ \varepsilon f_n, T(\hat{\Delta}(n))\Big\}.
\end{equation*}
Then for all $x\in\cA$,
\begin{equation*}
    \|x\|^2_{H_{S_{\varepsilon}(\hat{\Delta}(n))}^{-1}}\leq \max \Big\{\frac{\varepsilon^2}{f_n}, \frac{\hat{\Delta}_x^2(n)}{f_n}\Big\}.
\end{equation*}
\end{lemma}
This is Lemma 17 in the book of \citet{lattimore2017end}, where the reader can also find the proof.

\begin{lemma}\label{lem:cont}
Suppose that $T^m_x(\cdot)$ is uniquely defined at $\Delta$. Then it is continuous at $\Delta$.
\end{lemma}

\begin{proof}
Suppose it is not continuous. Then there exists a sequence $(\Delta_n)_{n=1}^\infty$ 
with $\lim_{n\to\infty} \norm{\Delta_n - \Delta} = 0$ and for which $\lim_{n\to\infty} T^m_x(\Delta_n) \neq T^m_x(\Delta)$ for some $m$ and $x \in \cA^m$.
Since $\Delta_n \to \Delta$ it follows that for sufficiently large $n$ the optimal actions with respect to $\Delta_n$ are the same as $\Delta$.
Hence, for sufficiently large $n$, by the definition of the optimisation problem,
\begin{align*}
T^m_{x^*_m}(\Delta_n) = \infty = T^m_{x^*_m}(\Delta) \,.
\end{align*}
Therefore there exists a context $m$ and suboptimal action $x \neq x^*_m$ such that $\lim_{n\to\infty} T^m_x(\Delta_n) \neq T^m_x(\Delta)$.
It is easy to check that the value of the optimisation problem is continuous. Specifically, that
\begin{align*}
\lim_{n\to\infty} \sum_{m=1}^M \sum_{x \in \cA^m} T^m_x(\Delta_n) = \sum_{m=1}^M \sum_{x \in \cA^m} T^m_x(\Delta)\,.
\end{align*}
Hence $\limsup_{n\to\infty} T^m_x(\Delta_n) < \infty$ for $x \neq x^*_m$. 
Therefore a compactness argument shows there exists a cluster point $S$ of the allocation $(T(\Delta_n))_{n=1}^\infty$ with $S^x_m \neq T^x_m(\Delta)$ for some
$m$ and $x \neq x^*_m$. And yet by the previous display
\begin{align*}
\sum_{m=1}^M \sum_{x \in \cA^m} S^m_x = \sum_{m=1}^M \sum_{x \in \cA^m} T^m_x(\Delta)\,.
\end{align*}
Since the constraints of the optimisation problem are continuous it follows that $S$ also satisfies the constraints in the optimisation problem and so
$S \neq T(\Delta)$ is another optimal allocation, contradicting uniqueness.
Therefore $T^m_x(\cdot)$ is continuous at $\Delta$.
\end{proof}

\section{Additional Experiments}\label{sec:add_exp}

In this section, we consider two more experiment settings in Figure \ref{plt_addition}.

\textbf{1. Small size action set.} We conduct the experiments with the number of action set equal to 5. Comparing with large size action set (Section \ref{sec:large}), we found that OAM still outperforms OSSB but the improvement is smaller, as one might expect.

\textbf{2. Randomly generated $\theta$.} For each replication, $\theta$ is randomly generated from multivariate normal with variance 10 and we normalise $\theta$ such that its $\ell_2$ norm is 1. OAM still outperforms OSSB for randomly generated $\theta$. In addition, we compare with the heuristic LinTS (remove all the variance blowup factors and use a Gaussian prior). We find that the heuristic LinTS enjoys the best performance by a modest margin. Analysing heuristic LinTS, however, remains a fascinating open problem. As far as we are aware, it is not known whether or not it even achieves sublinear regret in the worst case.

 \begin{figure}
	\centering
		\includegraphics[scale=0.35]{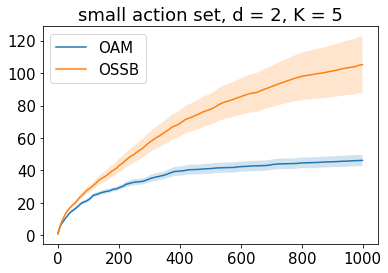}
		\includegraphics[scale=0.35]{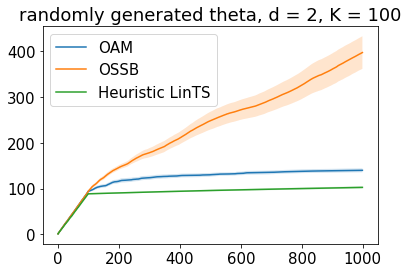}
		\vspace{-0.13in}
\caption{The left panel is for small size action set and the right panel is for randomly generated $\theta$. The results are averaged over $100$ realisations. }\label{plt_addition}
\end{figure}

\end{document}